\documentclass[lettersize,journal]{IEEEtran}
\usepackage{amsmath,amsfonts}
\usepackage{algorithmic}
\usepackage{algorithm}
\usepackage{array}
\usepackage[caption=false,font=normalsize,labelfont=sf,textfont=sf]{subfig}
\usepackage{textcomp}
\usepackage{stfloats}
\usepackage{url}
\usepackage{verbatim}
\usepackage{graphicx}
\usepackage{cite}
\usepackage{microtype}
\usepackage{graphicx}

\usepackage{hyperref}

\usepackage{url}
\usepackage{booktabs}       
\usepackage{amsfonts}       
\usepackage{nicefrac}       
\usepackage{microtype}      
\usepackage{xcolor}         
\usepackage{multirow}

\usepackage{amsmath}
\usepackage{amssymb}
\usepackage{mathtools}
\usepackage{amsthm}
\allowdisplaybreaks

\usepackage[capitalize,noabbrev]{cleveref}

\theoremstyle{plain}
\newtheorem{theorem}{Theorem}[section]

\newtheorem{lemma}[theorem]{Lemma}

\theoremstyle{definition}

\newtheorem{assumption}[theorem]{Assumption}
\theoremstyle{remark}

\begin{document}

\title{An Inclusive Theoretical Framework of Robust Supervised Contrastive Loss against Label Noise}


\author{Jingyi Cui, Yi-Ge Zhang, Hengyu Liu, Yisen Wang
\thanks{Jingyi Cui, Yi-Ge Zhang, and Yisen Wang are with the State Key Lab of General Artificial Intelligence, School of Intelligence Science and Technology, Peking University. Hengyu Liu is with the Department of Electronic Engineering, The Chinese University of Hong Kong.}
}



\maketitle

\begin{abstract}
Learning from noisy labels is a critical challenge in machine learning, with vast implications for numerous real-world scenarios. While supervised contrastive learning has recently emerged as a powerful tool for navigating label noise, many existing solutions remain heuristic, often devoid of a systematic theoretical foundation for crafting robust supervised contrastive losses. To address the gap, in this paper, we propose a unified theoretical framework for robust losses under the pairwise contrastive paradigm. In particular, we for the first time derive a general robust condition for arbitrary contrastive losses, which serves as a criterion to verify the theoretical robustness of a supervised contrastive loss against label noise. The theory indicates that the popular InfoNCE loss is in fact non-robust, and accordingly inspires us to develop a robust version of InfoNCE, termed Symmetric InfoNCE (SymNCE). Moreover, we highlight that our theory is an inclusive framework that provides explanations to prior robust techniques such as nearest-neighbor (NN) sample selection and robust contrastive loss. Validation experiments on benchmark datasets demonstrate the superiority of SymNCE against label noise.
\end{abstract}


\section{Introduction}

Supervised learning has demonstrated remarkable success across various machine learning domains, including computer vision \cite{krizhevsky2017imagenet, redmon2016you},  information retrieval \cite{zhang2016deep, onal2018neural}, and language processing \cite{howard2018universal, severyn2015twitter}.
Despite the recent success of unsupervised/self-supervised learning, supervised learning is still fundamentally important, especially on domain-specific data or when the sample size is small. 
Unfortunately, real-world data is often noisy, with mislabeled or wrongly labeled data points, which can significantly degrade the performance of a supervised model.
Consequently, learning with label noise becomes an important research problem in machine learning, and has been extensively studied in recent years. 

Recently, supervised and unsupervised contrastive representation learning \cite{khosla2020supervised, chen2020simple} has been introduced to the problem of learning from label noise.
Most methodological studies lean on the sample selection strategy, wielding contrastive learning as a mechanism to select confident samples based on the learned representations \cite{ortego2021multi, karim2022unicon, li2022selective, huang2023twin}. 
Specifically, MOIT \cite{ortego2021multi} employs a mixup-interpolated supervised contrastive learning whose learned embeddings are used to select clean samples through the nearest neighbor (NN) technique.
UNICON \cite{karim2022unicon} splits the dataset into clean and noisy subsets and treats the noisy sets as unlabeled with an unsupervised contrastive loss.
Sel-CL \cite{li2022selective} selects confident sample pairs by assessing the consistency between learned representations and labels, and leverages supervised contrastive learning to in turn improve the sample selection mechanism. 
TCL \cite{huang2023twin} identifies the wrongly labeled samples through a Gaussian mixture model (GMM) based on representations learned by a twin contrastive learning model.
Moreover, supervised contrastive learning with the sample selection mechanism is also widely used in other weakly supervised learning problems such as partial label learning \cite{wang2022pico, wang2023pico+,wen2021leveraged} and semi-supervised learning \cite{yang2022class, sun2024graph}.

Empirical results show that sample selection mechanisms based on supervised contrastive learning usually have significantly higher overall accuracy, especially when the noise rate is relatively low. 
Nonetheless, 
as the selected confident sample set may still contain a small fraction of label noise, the power of supervised contrastive learning fails to yield to the fullest, harming the overall model performances. 
Empirically, supervised contrastive learning significantly outperforms its unsupervised counterpart on clean datasets, whereas it is vulnerable to label noise. 
Therefore, it is of great significance to investigate robust contrastive learning frameworks against label noise.

In Figure \ref{fig::intro}, we compare the linear probing accuracy of supervised contrastive learning (SupCon) and unsupervised contrastive learning (SimCLR) on the benchmark CIFAR-10 and CIFAR-100 datasets under symmetric label noise. We show that SupCon outperforms SimCLR when the noise rate is low, whereas its performance drops severely as the noise rate increases. This result coincides with the empirical and theoretical findings in \cite{cui2023rethinking}. As a little spoiler, we also show that the robust contrastive loss function (SymNCE) proposed in this paper can significantly improve the performance of SupCon against label noise.

\begin{figure}[!h]
	\centering
	\subfloat[{\footnotesize CIFAR-10}]{
		\includegraphics[width=0.45\linewidth, trim = 0 0 0 0]{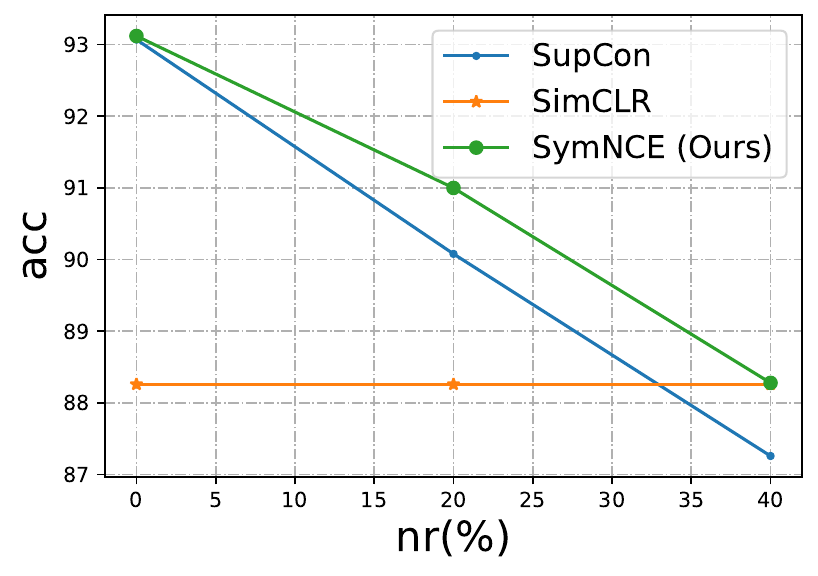}
		\label{fig::clnl_intro_10}
	}
	\subfloat[{\footnotesize CIFAR-100.}]{
		\includegraphics[width=0.45\linewidth, trim = 0 0 0 0]{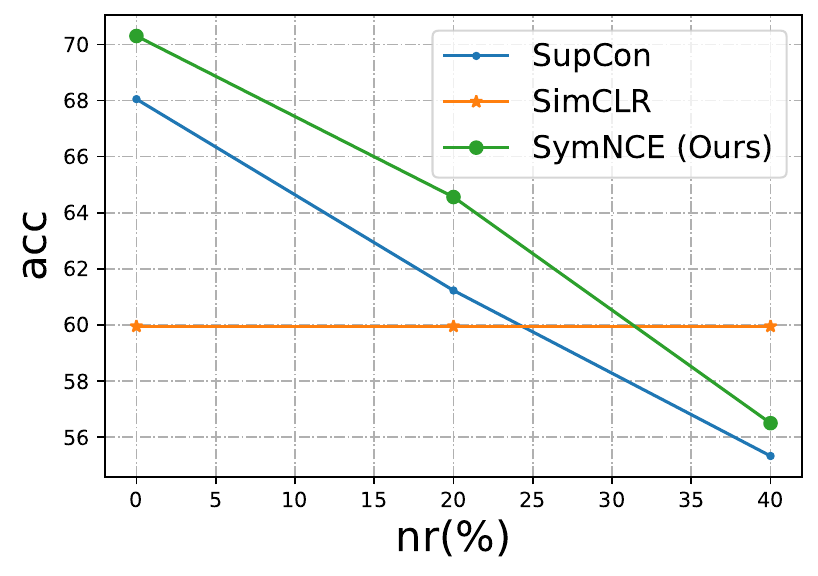}
		\label{fig::clnl_intro_100}
	}
	\caption{Comparisons in linear probing accuracy between supervised contrastive learning (SupCon), unsupervised contrastive learning (SimCLR), and robust supervised contrastive learning (SymNCE).}
	\label{fig::intro}
\end{figure}




Currently, research on robust supervised contrastive learning remains scarce. 
\cite{navaneet2022constrained} selects the $k$-nearest neighbors of the anchor sample as same-class positives to alleviate overfitting to label noise. Similar to the sample selection models for noisy label learning, this approach cannot guarantee that all selected samples are clean. Also, it lacks theoretical guarantees. 
\cite{cui2023rethinking} investigate the theoretical properties of supervised contrastive learning under label noise, but fails to propose robust contrastive learning methods.
Perhaps the most related to the study in this paper, RINCE \cite{chuang2022robust} proposed a robust loss function against noisy views which is proved to be a lower bound of Wasserstein Dependency Measure even under label noise. However, the assumption of noisy view generalization in its theoretical analysis is different from that of the standard label noise setting.
In short, there is still a lack of robust contrastive learning against label noise with sound theoretical guarantees.



Under such background, we aim to establish a comprehensive theoretical framework for robust supervised contrastive losses against label noise. Specifically, we first derive the robust condition for arbitrary pairwise contrastive losses, which serves as a criterion to verify whether a contrastive loss is robust. 
This general robust condition provides a unified understanding of existing robust contrastive methods and can inspire new robust loss functions.
Based on our theory, we show that the widely used InfoNCE loss fails to meet the general robust condition, and we propose its robust counterpart called \textbf{Sym}metric Info\textbf{NCE} (SymNCE) by adding a properly designed \textbf{Rev}erse Info\textbf{NCE} (RevNCE) to the InfoNCE loss function.

The contributions of this paper are summarized as follows.
\begin{itemize}
	\item We for the first time establish a unified theoretical framework for robust supervised contrastive losses against label noise. 
	We highlight that our theory serves as an inclusive framework that applies to existing robust contrastive learning techniques such as nearest neighbor (NN) sample selection and the RINCE loss \cite{chuang2022robust}.
	\item Inspired by our theoretical framework, we propose a robust counterpart of the widely used InfoNCE loss function called SymNCE by adding InfoNCE with RevNCE, which is a deliberately designed loss based on our derived robust condition. RevNCE helps InfoNCE to be robust and meanwhile functions similarly as InfoNCE, i.e. aligning the positive samples and pushing away the negatives.
	\item We empirically verify that our SymNCE loss is comparable and even outperforms existing state-of-the-art robust loss functions for learning with label noise.
\end{itemize}

This paper is organized as follows. In Section \ref{sec::relatedworks}, we introduce related works about noisy label learning and supervised contrastive learning. In Section \ref{sec::robust_theory}, we establish the theoretical framework of supervised contrastive learning under label noise, and present the main theorems, i.e., the robust condition of contrastive losses. Based on the main theorem, we demonstrate that InfoNCE is non-robust, and propose its robust counterpart called SymNCE in Section \ref{sec::SymNCE}. In Section \ref{sec::discussion}, we discuss that our theoretical result is an inclusive framework applying to both nearest-neighbor (NN) sample selection and RINCE. Finally, in Section \ref{sec::exp}, we conduct validation experiments to verify the theoretical findings.

\section{Related Works}\label{sec::relatedworks}

\subsection{Learning with Label Noise} 
The major approaches of learning with label noise include robust architecture \cite{goldberger2017training, han2018masking, yao2018deep}, robust regularization \cite{lukasik2020does, pereyra2017regularizing, wei2021open}, sample selection \cite{han2018co, song2019selfie, yu2019does,wang2018iterative}, loss adjustment \cite{hendrycks2018using,ma2018dimensionality,patrini2017making, zhang2021approximating}, and robust loss function \cite{ghosh2017robust, ma2020normalized, wang2019symmetric, zhang2018generalized}. 
While robust architecture adds noise adaptation layer or dedicated architectures to the deep neural networks, robust regularization explicitly or implicitly regularizes the model complexity and prevents overfitting to the noisy labels. Sample selection adopts specific techniques or network designs to identify clean samples. Though these methodologies have exhibited empirical success, they often lean on intricate designs and can sometimes appear heuristic in nature. Loss adjustment modifies the loss function during training, whereas the effectiveness relies heavily on the estimated noise transition matrix. 

Diverging from the aforementioned strategies, robust loss functions typically come with the assurance of theoretical robustness and effectiveness. For instance, \cite{ghosh2017robust} first theoretically proved the general robust condition for supervised classification losses, and showed that the Mean Absolute Error (MAE) loss is robust whereas Cross Entropy (CE) is not. 
However, as MAE performs poorly on complex datasets, \cite{zhang2018generalized} generalized the MAE and CE losses, and proposed the Generalized Cross Entropy (GCE) loss by applying the Box-Cox transformation to the probabilistic predictions. 
\cite{wang2019symmetric} proposed the Symmetric Cross Entropy (SCE) loss by combining the CE loss with the provably robust Reverse Cross Entropy (RCE) loss. 
\cite{ma2020normalized} showed that any supervised loss can be robust to noisy labels by a simple normalization, and proposed the Active Passive Loss (APL) by combining two robust loss functions.

\subsection{Supervised Contrastive Learning}

Contrastive learning algorithms \cite{chen2020simple, he2020momentum,wang2021residual} were first proposed for self-supervised representation learning, where unsupervised contrastive losses pulls together an anchor and its augmented views in the embedding space. 
\cite{khosla2020supervised} extended contrastive learning to supervised training and proposed the Supervised Contrastive (SupCon) loss which takes same-class augmented examples as the positive labels within the InfoNCE loss. 
SupCon achieves significantly better performance than the state-of-the-art CE loss, especially on large-scale datasets.

\textbf{Supervised contrastive learning with label noise.}
Recently, supervised contrastive learning has been introduced to solve weakly supervised learning problems such as noisy label learning \cite{huang2023twin, li2022selective, ortego2021multi, yan2022noise, yang2022class}. For example, \cite{ortego2021multi} proposed the interpolated contrastive learning which adopts the interpolated samples in the contrastive loss to avoid overfitting to noisy labels, whose learned embeddings are then used to select clean samples through the nearest neighbor (NN) technique.
\cite{li2022selective} not only selected confident samples for classifier training with the NN technique, but also mined positive pairs from the confident same-class samples for the training of contrastive representation learning.
\cite{huang2023twin} proposed the twin contrastive learning model to learn robust representations, where the wrongly labeled samples are recognized as the out-of-distribution samples through a Gaussian mixture model (GMM).
\cite{yan2022noise} leveraged the negative correlations from the noisy data to avoid same-class negatives for contrastive learning, whereas the positive sample selection procedure remains unsupervised.
\cite{yang2022class} proposed to use supervised contrastive learning for semi-supervised learning, where the learned representations are used to refine pseudo labels.

\textbf{Robustness of contrastive learning against label noise.}
Aside from the empirical success of supervision contrastive learning with label noise, relatively few works focus on robust supervised contrastive learning, especially those with theoretical guarantees.
As a byproduct in \cite{navaneet2022constrained}, constraining the same-class positives in contrastive learning within the $k$-nearest neighbors of the anchor sample empirically alleviates overfitting to label noise. Yet this result lacks theoretical guarantees. 
Inspired by the robust condition for supervised classification losses in \cite{ghosh2017robust}, \cite{chuang2022robust} proposed the robust InfoNCE (RINCE) loss function against noisy views which is proved to be a lower bound of Wasserstein Dependency Measure even under label noise. However, the assumption of noisy view generalization in its theoretical analysis is different from that of the standard label noise setting.
On the other hand, some theoretical studies focus on the unsupervised contrastive pretraining approaches, and prove the robustness of downstream classifiers with features learned by self-supervised contrastive learning \cite{cheng2021demystifying, xue2022investigating} without using noisy labels in pre-training stage. 

To summarize, while supervised contrastive learning has demonstrated immense empirical promise in navigating label noise, the design and theoretical investigation of robust contrastive losses is still under-exploited.
To fill in the blank, in this paper, we propose a unified theoretical framework for robust contrastive losses. We derive a general robust condition for arbitrary contrastive losses, inspired by which we further propose a robust counterpart of the widely used InfoNCE loss called SymNCE.

\section{Robust Condition for Contrastive Losses}\label{sec::robust_theory}

In this section, we first introduce some preliminaries for the mathematical formulation of the risks for contrastive losses and label noise. Then we propose the formal formulation of contrastive risks under the distribution of label noise. After that, we can derive a general robust condition for arbitrary contrastive losses, and discuss the relationship between our theoretical result and related works. 
All proofs are shown in Section \ref{sec::proof1}.

\subsection{Preliminaries}\label{sec::prelim}

\textbf{Notations.} 
Suppose that random variables $X \in \mathcal{X}$ and $c \in [C]:=\{1,\ldots,C\}$. Let the input data $\{(x_i,y_i)\}_{i\in[n]}$ be i.i.d.~sampled from the joint distribution $\mathrm{P}(X,c)$. 
For $i \in [C]$, we denote the marginal distribution $\pi_i = \mathrm{P}(c=i)$, and 
denote class conditional distribution $\rho_i = \mathrm{P}(x|c=i)$.
Under the noisy label distribution, we denote $\tilde{c} \in [C]$ as the random variable of noisy label, and let the noisy input data $\{(x_i,\tilde{y}_i)\}_{i\in[n]}$ be i.i.d.~sampled from the joint distribution $\mathrm{P}(X,\tilde{c})$. 
For $i \in [C]$, we denote $\tilde{\pi}_i = \mathrm{P}(\tilde{c}=i)$ and $\tilde{\rho}_i = \mathrm{P}(x|\tilde{c}=i)$ as the noisy marginal and class conditional distributions, respectively.
For notational simplicity, we denote $\boldsymbol{\pi} = (\pi_i)_{i \in [C]}$ and $\tilde{\boldsymbol{\pi}} = (\tilde{\pi}_i)_{i \in [C]}$.


\textbf{Mathematical Formulations of Supervised Contrastive Risk.}
We first generalize the supervised contrastive learning loss proposed in \cite{khosla2020supervised} to arbitrary contrastive losses $\mathcal{L}(x, \{x_m^+\}_{m=1}^M, \{x_k^-\}_{k=1}^K; f)$, where $x$ is the anchor sample, $\{x_m^+\}_{m=1}^M$ are $M$ positive samples, $\{x_k^-\}_{k=1}^K$ are $K$ negative samples, and $f : \mathcal{X} \to \mathbb{R}^d$ denotes the representation function. For notational simplicity, we write $\mathcal{L}(\boldsymbol{x}; f)$ instead in the rest of this paper.

For the mathematical formulations, we follow the CURL framework \cite{arora2019theoretical} for the mathematical formulation of supervised contrastive learning under clean labels (without label noise). Note that CURL uses the concept of latent classes to describe the distribution of positive pairs. In supervised contrastive learning, we naturally assume the latent classes to be the annotated classes, since positive samples are selected as those with the same labels. 

The generation process of positive and negative samples is described as follows: (i) draw positive and negative classes $c_+, \{c^-_k\}_{k \in [K]} \sim \boldsymbol{\pi}^{K+1}$; (ii) draw an anchor sample $x$ from class $c^+$ with probability $\rho_{c^+} = \mathrm{P}(x|c = c^+)$; (iii) draw $M$ positive samples $\{x_m^+\}_{m=1}^M$ from class $c^+$ with probability $\rho_{c^+}$; and (iv) for $k \in [K]$, draw negative sample $x_k^-$ from class $c_k^-$ with probability $\rho_{c_k^-}$.

Then the corresponding risk $\mathcal{R}(\mathcal{L}; f)$ for loss $\mathcal{L}(\boldsymbol{x}; f)$ is formulated as
{
\begin{align}\label{eq::cleanrisk}
        \mathbb{E}_{c^+, \{c_k^-\}_{k=1}^K \sim \boldsymbol{\pi}^{K+1}}
	\mathbb{E}_{x,\{x_m^+\}_{m=1}^M \sim \rho_{c^+},\ 
		x^-_k \sim \rho_{c_k^-}, k \in [K]} 
	\mathcal{L}(\boldsymbol{x}; f).
\end{align}
}


\textbf{Label Noise Assumptions.}
We model the generation process of label noise through label corruption. We assume that the label corruption process is conditionally dependent of the true label and independent of data features.
To be specific, denote $\tilde{c} \in [C]$ as the random variable of noisy label, and we use $q_j(i):= \mathrm{P}(\tilde{c}=i | c=j, x)$ to denote the probability of true label $j \in [C]$ corrupted to label $i \in [C]$.
Then according to the law of total probability, the posterior probability of noisy labels can be expressed as $\mathrm{P}(\tilde{c}=i|x) = \sum_{j=1}^C q_j(i) \mathrm{P}(c=j|x)$.

Next, we take the classic symmetric label noise assumption as an example to formulate label noise. This assumption is common and widely adopted in the community of learning with label noise \cite{ghosh2017robust, ma2020normalized, natarajan2013learning, wang2019symmetric}. The discussions about asymmetric label noise can be found in the Appendix \ref{sec::proof1}.

\begin{assumption}[Symmetric label noise]\label{ass::symnoise}
	For noise rate $\gamma \in (0,1)$, we assume that
	there holds
	$q_i(i)=1-\gamma$ and $q_i(j)=\gamma/(C-1)$ for all $j\neq i$.
\end{assumption}

In Assumption \ref{ass::symnoise}, we assume that the label is corrupted to a noisy label with probability $\gamma$ and remains clean with probability $1-\gamma$, where a label has uniform probability of corrupting to any of the other labels.

\subsection{Supervised Contrastive Risks under Label Noise}\label{sec::noisyrisk}

For supervised contrastive learning under label noise, the positive pairs are selected according to the (corrupted) annotated labels. Therefore, we combine the framework of supervised contrastive risk and the label noise assumptions shown in the previous section to formulate the supervised contrastive risk under label noise. 

The generation process of positive and negative samples under label noise can be described as follows: (i) draw noisy positive and negative classes $\tilde{c}_+, \{\tilde{c}^-_k\}_{k \in [K]} \sim \tilde{\boldsymbol{\pi}}^{K+1}$; (ii) draw an anchor sample $x$ from class $\tilde{c}^+$ with probability $\tilde{\rho}_{\tilde{c}^+} = \mathrm{P}(x|\tilde{c} = \tilde{c}^+)$; (iii) draw $M$ positive samples $\{x_m^+\}_{m=1}^M$ from class $\tilde{c}^+$ with probability $\tilde{\rho}_{\tilde{c}^+}$; and (iv) for $k \in [K]$, draw negative sample $x_k^-$ from class $\tilde{c}_k^-$ with probability $\tilde{\rho}_{\tilde{c}_k^-}$.

Finally, we can formulate the noisy risk $\widetilde{\mathcal{R}}(\mathcal{L}; f)$ for loss $\mathcal{L}(\boldsymbol{x}; f)$ as
{
\begin{align}\label{eq::noisyrisk}
        \mathbb{E}_{\tilde{c}^+, \{\tilde{c}_k^-\}_{k=1}^K \sim \tilde{\boldsymbol{\pi}}^{K+1}}
        \mathbb{E}_{x,\{x_m^+\}_{m=1}^M \sim \tilde{\rho}_{\tilde{c}^+},\ 
		x^-_k \sim \tilde{\rho}_{\tilde{c}_k^-}, k \in [K]}
	\mathcal{L}(\boldsymbol{x}; f).
\end{align}
}

Note that in \cite{khosla2020supervised}, the supervised contrastive loss can be expressed as 
{
\begin{align*}
	\mathcal{L}(x, \{x_m^+\}_{m=1}^M, \{x_k^-\}_{k=1}^K; f) 
	= \frac{1}{M} \sum_{i=1}^M \mathcal{L}_m (x, x_m^+, \{x_k^-\}_{k=1}^K; f),
\end{align*}
}
and thus due to the conditional independence of $x_m^+ | x$, we demonstrate the equivalence between the risks for single and multiple positive samples, i.e.
{
\begin{align}
	\mathcal{R}(\mathcal{L}; f) 
	&= \mathbb{E}_{x,\{x_m^+\}_{m=1}^M,\{x^-_k\}_{k=1}^K} \frac{1}{M} \sum_{i=1}^M \mathcal{L}_m (x, x_m^+, \{x_k^-\}_{k=1}^K; f)
	\nonumber\\
	&= \frac{1}{M} \sum_{i=1}^M  \mathbb{E}_{x, x_m^+, \{x^-_k\}_{k=1}^K} \mathcal{L}_m (x, x_m^+, \{x_k^-\}_{k=1}^K; f)
	\nonumber\\
	&=  \frac{1}{M} \sum_{i=1}^M \mathcal{R}(\mathcal{L}_m; f)
	= \mathcal{R}(\mathcal{L}_m; f).
\end{align}
}
Therefore, in the following, we analyze the case with a single positive sample without loss of generality. 

\subsection{Important Lemmas}

Before we proceed, we introduce two important lemmas.
First, in Lemma \ref{lem::pi}, we derive the relationship between the clean and noisy distributions under the label corruption assumption.

\begin{lemma}\label{lem::pi}
	Under label-dependent label corruption, we have 
	$\tilde{\pi}_i = \sum_{j=1}^C q_j(i)\pi_j$, and 
	$\tilde{\rho}_i(x) = \big(\sum_{j=1}^C q_j(i) \rho_j(x)\pi_j\big)/\big(\sum_{j=1}^C q_j(i)\pi_j\big)$.
\end{lemma}

\begin{proof}[Proof of Lemma \ref{lem::pi}]
	Under label-dependent label corruption, by the law of total probability, there holds
	\begin{align}\label{eq::Atransition}
		\mathrm{P}(\tilde{c} = i | x) 
		&= \sum_{i=1}^C \mathrm{P}(\tilde{c} = i | c = j, x) \mathrm{P}(c = j | x)
            \nonumber\\
		&= \sum_{i=1}^C q_j(i) \mathrm{P}(c = j | x).
	\end{align}
	By taking expectation of $x$ on both sides of \eqref{eq::Atransition}, we have
	\begin{align}\label{eq::Api}
		\tilde{\pi}_i = \mathrm{P}(\tilde{c} = i) = \sum_{i=1}^C q_j(i) \mathrm{P}(c = j) = \sum_{i=1}^C q_j(i) \pi_j.
	\end{align}
	On the other hand, by Bayes' theorem, there holds
	\begin{align}
		\tilde{\rho}_i(x) &= \mathrm{P}(x | \tilde{c} = i)
		= \frac{\mathrm{P}(\tilde{c} = i | x) \mathrm{P}(x)}{\mathrm{P}(\tilde{c} = i)}.
	\end{align}
	Then combining with \eqref{eq::Atransition} and \eqref{eq::Api}, we have
 {
	\begin{align*}
		\tilde{\rho}_i(x) 
		&= \frac{\sum_{i=1}^C q_j(i) \mathrm{P}(c = j | x)\mathrm{P}(x)}{\sum_{i=1}^C q_j(i) \mathrm{P}(c = j)}
		\nonumber\\
		&= \frac{\sum_{i=1}^C q_j(i) \mathrm{P}(x | c = j)\mathrm{P}(c=j)}{\sum_{i=1}^C q_j(i) \mathrm{P}(c = j)}
		\nonumber\\
		&= \frac{\sum_{i=1}^C q_j(i) \rho_j(x)\pi_j}{\sum_{i=1}^C q_j(i)\pi_j}.
	\end{align*}
 }
\end{proof}

In Lemma \ref{lem::decom}, we decompose the noisy risk for arbitrary contrastive losses without assuming any specific label types.

\begin{lemma}\label{lem::decom}
	For arbitrary contrastive loss function $\mathcal{L}(\boldsymbol{x}; f)$, the noisy contrastive risk with $M=1$ can be decomposed into
 {
	\begin{align}
		\widetilde{\mathcal{R}}(\mathcal{L}; f) 
		&= \sum_{i \in [C]} 
		a(i)^{-1} \sum_{u\in[C]} \sum_{u^+\in[C]} \pi_u \pi_{u^+} 	q_u(i) q_{u^+}(i) 
            \nonumber\\
		&\quad\cdot\mathbb{E}_{\substack{x \sim \rho_u\\x^+ \sim \rho_{u^+}}}
		\mathbb{E}_{\{x_k^-\}_{k=1}^K \overset{\text{i.i.d.}}{\sim} \mathrm{P}_X} \mathcal{L}(\boldsymbol{x}; f),
	\end{align}
 }
	where 
	$a(i) := \sum_{u\in[C]} q_u(i) \pi_u$.
\end{lemma}

For conciseness, we only present the proof sketch here and defer the complete proof to Appendix \ref{sec::proof1}.
\begin{proof}[Proof Sketch of Lemma \ref{lem::decom}]
    	By definition of the noisy risk in \eqref{eq::noisyrisk}, there holds
 {
	\begin{align}\label{eq::tildeR11}
		&\widetilde{\mathcal{R}}(\mathcal{L}; f) 
            \nonumber\\
		&= \mathbb{E}_{\tilde{c}^+, \{\tilde{c}_k^-\}_{k=1}^K \sim \tilde{\boldsymbol{\pi}}^{K+1}} \mathbb{E}_{\substack{x, x^+ \sim \tilde{\rho}_{\tilde{c}^+}\\
				x^-_k \sim \tilde{\rho}_{\tilde{c}_k^-}, k \in [K]}} 
		\mathcal{L}(x, x^+, \{x_k^-\}_{k=1}^K; f)
		\nonumber\\
		&= \sum_{i, j_1, \ldots, j_K \in [C]} \tilde{\pi}_i \tilde{\pi}_{j_1} \cdots \tilde{\pi}_{j_K} 
            \nonumber\\
            &\qquad\cdot
		\mathbb{E}_{\substack{x, x^+ \sim \tilde{\rho}_{i}\\
				x^-_k \sim \tilde{\rho}_{j_k}, k \in [K]}} 
		\mathcal{L}(x, x^+, \{x_k^-\}_{k=1}^K; f)
		\nonumber \\
		&= \sum_{i, j_1, \ldots, j_K \in [C]} \sum_{m \in [C]} q_m(i) \pi_m 
            \nonumber\\
            &\qquad\cdot\sum_{l_1 \in [C]} q_{l_1}(j_1) \pi_{l_1} \cdots \sum_{l_K \in [C]} q_{l_K}(j_K) \pi_{l_K}
		\nonumber\\
		&\qquad\cdot
		\mathbb{E}_{\substack{x, x^+ \sim \tilde{\rho}_{i}\\
				x^-_k \sim \tilde{\rho}_{j_k}, k \in [K]}} 
		\mathcal{L}(x, x^+, \{x_k^-\}_{k=1}^K; f).
	\end{align}
 }
	Denoting $a(i) := \sum_{u\in[C]} q_u(i) \pi_u$,
        $A(u,u^+,\boldsymbol{v}) := \mathbb{E}_{\substack{x \sim \rho_u, x^+ \sim \rho_{u^+}\\
				x^-_k \sim \rho_{v_k}, k \in [K]}} 
		\mathcal{L}(x, x^+, \{x_k^-\}_{k=1}^K; f)$,
        and plugging Lemma \ref{lem::pi} in \eqref{eq::tildeR11}, with some calculations we have 
 {
	\begin{align}\label{eq::tildeR21}
		&\mathbb{E}_{\substack{x, x^+ \sim \tilde{\rho}_{i}\\
				x^-_k \sim \tilde{\rho}_{j_k}, k \in [K]}} 
		\mathcal{L}(x, x^+, \{x_k^-\}_{k=1}^K; f)
		\nonumber\\
		&= \idotsint \tilde{\rho}_i(x) \tilde{\rho}_i(x^+) \tilde{\rho}_{j_1}(x_1^-) \cdots \tilde{\rho}_{j_K}(x_K^-) 
            \nonumber\\
            &\ \cdot \mathcal{L}(x, x^+, \{x_k^-\}_{k=1}^K; f) \, dx \, dx^+ \, dx^-_1 \ldots dx^-_K
		\nonumber\\
		&= (a(i)^2 a(j_1) \ldots a(j_K))^{-1} 
		\nonumber\\
		&\cdot \sum_{u, u^+, v_1, \ldots, v_K \in[C]} \pi_u \pi_{u^+} \pi_{v_1} \cdots \pi_{v_K} 
		\nonumber\\ 	
            &\ \cdot q_u(i) q_{u^+}(i) q_{v_1}(j_1) \cdots q_{v_K}(j_K) A(u,u^+,\boldsymbol{v}).
	\end{align}
 }
	Combining \eqref{eq::tildeR11} and \eqref{eq::tildeR21}, we have
 {
	\begin{align}
		&\widetilde{\mathcal{R}}(\mathcal{L}; f) 
            \nonumber\\
		&= \sum_{i, j_1, \ldots, j_K \in [C]} (a(i)^2 a(j_1) \ldots a(j_K))^{-1} 
		\nonumber\\
		&\qquad\cdot \sum_{m, l_1, \ldots, l_K \in [C]} q_m(i) q_{l_1}(j_1) \cdots q_{l_K}(j_K) \cdot \pi_m \pi_{l_1} \cdots \pi_{l_K} 
		\nonumber\\
		&\qquad\cdot \sum_{u, u^+, v_1, \ldots, v_K \in[C]} q_u(i) q_{u^+}(i) q_{v_1}(j_1) \cdots q_{v_K}(j_K) 
		\nonumber\\ 
		&\qquad\cdot \pi_u \pi_{u^+} \pi_{v_1} \cdots \pi_{v_K} A(u,u^+,\boldsymbol{v})
		\nonumber\\
		&= \sum_{i \in [C]}
		a(i)^{-2} \sum_{m \in [C]} \sum_{u\in[C]} \sum_{u^+\in[C]} \pi_m \pi_u \pi_{u^+} q_m(i) q_u(i) q_{u^+}(i) 
		\nonumber\\
		&\qquad\cdot \sum_{j_1 \in [C]} a(j_1)^{-1} \sum_{l_1 \in [C]} \sum_{v_1\in[C]} \pi_{l_1} \pi_{v_1} q_{l_1}(j_1) q_{v_1}(j_1) \cdots 
		\nonumber\\
		&\qquad\cdot \sum_{j_K \in [C]}
		a(j_K)^{-1} \sum_{l_K \in [C]} \sum_{v_K\in[C]} \pi_{l_K} \pi_{v_K} q_{l_K}(j_K) q_{v_K}(j_K)  
            \nonumber\\
		&\qquad\cdot A(u,u^+,\boldsymbol{v}).
		\label{eq::Anegs}
	\end{align}
 }
	Then by separately calculating the positive and negative terms, we have
        {
	\begin{align}
		&\widetilde{\mathcal{R}}(\mathcal{L}; f) 
		\nonumber\\
            &= \sum_{i \in [C]} 
		a(i)^{-1} \sum_{u\in[C]} \sum_{u^+\in[C]} \pi_u \pi_{u^+} q_u(i) q_{u^+}(i) 
		\nonumber\\
		&\qquad\cdot\sum_{v_1, \ldots, v_K \in[C]} \pi_{v_1} \cdots \pi_{v_K} A(u,u^+,\boldsymbol{v})
		\nonumber\\
		&=\sum_{i \in [C]} 
		a(i)^{-1} \sum_{u\in[C]} \sum_{u^+\in[C]} \pi_u \pi_{u^+} 	q_u(i) q_{u^+}(i) 
		\nonumber\\
		&\qquad\cdot\mathbb{E}_{\substack{x \sim \rho_u,\\x^+ \sim \rho_{u^+}}}
		\mathbb{E}_{\{x_k^-\}_{k=1}^K \overset{\text{i.i.d.}}{\sim} \mathrm{P}_X} \mathcal{L}(\boldsymbol{x}; f).
	\end{align}
	}
\end{proof}

\textbf{Remark.} By Lemma \ref{lem::decom}, we see that the label noise only affects the noisy contrastive risk by changing the distribution of positive samples. The negative sample distribution remains unaffected because the negatives are uniformly drawn from the data distribution regardless of their labels. 
On the contrary, for positive samples selected from the noisy distribution, it is probable that a positive pair actually has different ground truth labels, and accordingly making the model overfitting to label noise.

\subsection{Robust Condition for Contrastive Losses}\label{sec::thoerem}

Compared with the contrastive risk under clean label distribution $\mathcal{R}(\mathcal{L}; f)$ in \eqref{eq::cleanrisk}, the noisy contrastive risk $\widetilde{\mathcal{R}}(\mathcal{L}; f)$ in \eqref{eq::noisyrisk} suffers from additional error because of label corruption. 
Specifically, under the noisy label distribution, the positive samples have the same annotated labels, but their true labels can be different. Aligning such noisy positive labels harms the representation learning procedure. 
Therefore, in this part, we first separate the additional risk caused by label corruption from the clean contrastive risk by explicitly deriving the relationship between clean and noisy risks. Then we propose the robust condition for contrastive losses.

In Theorem \ref{thm::riskconsistency}, we show the relationship between the noisy contrastive risk $\widetilde{\mathcal{R}}(\mathcal{L}; f)$ and the clean contrastive risk $\mathcal{R}(\mathcal{L}; f)$. 
\begin{theorem}\label{thm::riskconsistency}
	Assume that the input data is class balanced, i.e. $\pi_i = 1/C$ for $i \in [C]$. Then under Assumption \ref{ass::symnoise}, for an arbitrary contrastive loss $\mathcal{L}(\boldsymbol{x}; f)$, 
	there holds
		\begin{align}\label{eq::riskconsistency}
			\widetilde{\mathcal{R}}(\mathcal{L}; f) 
			&= \Big(1-C/(C-1)\gamma\Big)^2 \mathcal{R}(\mathcal{L}; f) 
			\nonumber\\
			&+ C/(C-1)\gamma \cdot \Big(2-C/(C-1)\gamma\Big) \Delta \mathcal{R}(\mathcal{L}; f),
		\end{align}
	where 
	{
	\begin{align*}
		\Delta \mathcal{R}(\mathcal{L}; f) 
		&:= \mathbb{E}_{c^+, \{c_m^+\}_{m=1}^M, \{c_k^-\}_{k=1}^K \sim \boldsymbol{\pi}^{M+K+1}} 
		\nonumber\\
		&\quad\cdot\mathbb{E}_{x \sim \rho_{c^+},\ x_m^+ \sim \rho_{c_m^+}, m \in [M],\ x^-_k \sim \rho_{c_k^-}, \, k \in [K]}
		\mathcal{L}(\boldsymbol{x}; f).
	\end{align*}
	}
\end{theorem}

\begin{proof}[Proof Sketch of Theorem \ref{thm::riskconsistency}]
        Denote $B(u,u^+) := \sum_{v_1, \ldots, v_K \in[C]} \pi_{v_1} \cdots \pi_{v_K} A(u,u^+,\boldsymbol{v})$. Then combining Assumption \ref{ass::symnoise} and Lemma \ref{lem::decom}, by decomposing the summation into
        \begin{align}
            \sum_{u\in[C]} \sum_{u^+\in[C]} 
            = \sum_{u=u^+=i} + \sum_{u=i} \sum_{u^+ \neq i} + \sum_{u \neq i} \sum_{u^+=i} + \sum_{u \neq i} \sum_{u^+ \neq i},
        \end{align}
        we can calculate that
        {
	\begin{align}
		&\widetilde{\mathcal{R}}(\mathcal{L}; f) 
		\nonumber\\
		&= \sum_{i \in [C]} 
		a(i)^{-1} \sum_{u\in[C]} \sum_{u^+\in[C]} \pi_u \pi_{u^+} q_u(i) q_{u^+}(i) B(u,u^+)
		\nonumber\\
		&= \sum_{u \in [C]} a(u)^{-1} \Big(1-\frac{C}{C-1}\gamma\Big)^2 \pi_u^2 B(u,u) 
		\nonumber\\
		&+ \sum_{u \in [C]} a(u)^{-1} \frac{\gamma}{C-1}\Big(1-\frac{C}{C-1}\gamma\Big)\pi_u \sum_{u^+ \in [C]} \pi_{u^+} B(u,u^+)
		\nonumber\\
		&+ \sum_{u^+ \in [C]} a(u^+)^{-1} \frac{\gamma}{C-1}\Big(1-\frac{C}{C-1}\gamma\Big)\pi_{u^+} \sum_{u \in [C]} \pi_{u} B(u, u^+)
		\nonumber\\
		&+ \sum_{i \in [C]} a(i)^{-1} \frac{\gamma^2}{(C-1)^2} \sum_{u \in [C]} \sum_{u^+ \in [C]} \pi_u \pi_{u^+} B(u,u^+).
	\end{align}
 }
	When the input data is class balanced, i.e. $\pi_i=\frac{1}{C}$ for $i \in [C]$, we have for $i \in [C]$,
		$a(i) = \sum_{u\in[C]}q_u(i)\pi_u = \frac{1}{C} \sum_{u\in[C]}q_u(i) = \frac{1}{C}$.
	Then we have
        {
	\begin{align}
		&\widetilde{\mathcal{R}}(\mathcal{L}; f) 
		\nonumber\\
		&= \Big(1-\frac{C}{C-1}\gamma\Big)^2 \sum_{u \in [C]} \pi_u B(u,u) 
		\nonumber\\
		&+ \frac{C\gamma}{C-1}\Big(1-\frac{C}{C-1}\gamma\Big) \sum_{u \in [C]} \pi_u \sum_{u^+ \in [C]} \pi_{u^+} B(u,u^+)
		\nonumber\\
		&+ \frac{C\gamma}{C-1}\Big(1-\frac{C}{C-1}\gamma\Big) \sum_{u^+ \in [C]} \pi_{u^+} \sum_{u \in [C]} \pi_{u} B(u, u^+)
		\nonumber\\
		&+ \frac{C^2\gamma^2}{(C-1)^2} \sum_{u \in [C]} \sum_{u^+ \in [C]} \pi_u \pi_{u^+} B(u,u^+)
		\nonumber\\
		&= \Big(1-\frac{C}{C-1}\gamma\Big)^2 \sum_{u \in [C]} \pi_u B(u,u) 
		\nonumber\\
		&+ \frac{C\gamma}{C-1}\Big(2-\frac{C}{C-1}\gamma\Big) \sum_{u,u^+ \in [C]} \pi_u \pi_{u^+} B(u,u^+)
	\end{align}
 }
	Note that 
        {
	\begin{align}
		&\sum_{u \in [C]} \pi_u B(u,u) 
		= \mathcal{R}(\mathcal{L}; f),
	\end{align}
        }
	and
        {
	\begin{align}
		&\sum_{u,u^+ \in [C]} \pi_u \pi_{u^+} B(u,u^+)
		= \Delta \mathcal{R}(\mathcal{L}; f).
	\end{align}
        }
	Thus we complete the proof.
\end{proof}



\textbf{Remark.} From Theorem \ref{thm::riskconsistency}, we show that the noisy risk is a linear combination of the clean risk $\mathcal{R}(\mathcal{L}; f)$ and the additional risk $\Delta \mathcal{R}(\mathcal{L}; f)$.
Specifically, when $\gamma=0$, i.e. the labels are clean, the RHS of \eqref{eq::riskconsistency} degenerates to $\mathcal{R}(\mathcal{L}; f)$.
Note that in the additional risk term $\Delta \mathcal{R}(\mathcal{L}; f)$, the anchor sample $x$ and positive samples $\{x_m^+\}_{m=1}^{M}$ are independently and uniformly sampled from the input distribution. 
The additional risk $\Delta \mathcal{R}(\mathcal{L}; f)$ represents the negative influence of label corruption, because no feature information can be learned by minimizing $\Delta \mathcal{R}(\mathcal{L}; f)$ and aligning such independent positive samples.

In noisy label learning, the goal is to optimize the clean risk $\mathcal{R}(f)$, whereas we can only achieve the noisy risk $\widetilde{\mathcal{R}}(f)$
since the clean distribution remains unknown. 
Nonetheless, if the additional risk $\Delta \mathcal{R}(\mathcal{L}; f)$ is a constant, then minimizing the noisy and clean risks results in the same optimal representation function $f$. 

\begin{theorem}\label{thm::noisetolerant}
	Assume that the input data is class balanced, and there exists a constant $A \in \mathbb{R}$ such that $\Delta \mathcal{R}(\mathcal{L}; f) = A$. Then under Assumption \ref{ass::symnoise}, contrastive loss $\mathcal{L}$ is noise tolerant if $\gamma < \frac{C-1}{C}$.
\end{theorem}

\begin{proof}[Proof of Theorem \ref{thm::noisetolerant}]
	For symmetric label noise, by Theorem \ref{thm::riskconsistency} and that $\Delta \mathcal{R}(\mathcal{L}; f) = A$, we have
        {
	\begin{align*}
		\widetilde{\mathcal{R}}(\mathcal{L}; f) 
		&= \Big(1-\frac{C}{C-1}\gamma\Big)^2 \mathcal{R}(\mathcal{L}; f) 
		+ \frac{C\gamma}{C-1}\Big(2-\frac{C}{C-1}\gamma\Big) A.
	\end{align*}
        }
	Suppose $f^*$ is the minimizer of $\widetilde{\mathcal{R}}(\mathcal{L}; f)$, i.e. for all $f$
	\begin{align}
		\widetilde{\mathcal{R}}(\mathcal{L}; f^*) \leq \widetilde{\mathcal{R}}(\mathcal{L}; f).
	\end{align}
	Then if $\gamma \leq (C-1)/C$, we have
	\begin{align}
		\mathcal{R}(\mathcal{L}; f^*)
		\leq \mathcal{R}(\mathcal{L}; f),
	\end{align}
	that is, $f^*$ is also the minimizer of $\mathcal{R}(\mathcal{L}; f)$.
\end{proof}


\textbf{Remark.} The symmetrization result in Theorem \ref{thm::noisetolerant} may look similar to that in previous works for classification losses, but looking into the details, they are totally different. Firstly, the notion of "symmetric" in the robust condition is different. The symmetric condition in previous works, e.g. \cite{ghosh2017robust}, requires the expectation of $\mathcal{L}(g(x), c)$ w.r.t. all classes is a constant, i.e. $\mathcal{L}(g(x), c)$ is ``symmetric'' over all classes, whereas we require $\Delta\mathcal{R}(\mathcal{L};f)$* to be a constant, i.e. we require the contrastive loss $\mathcal{L}(\boldsymbol{x};f)$ to be ``symmetric'' over all positive samples. Since there is no such notion of positive pairs in the classification losses, the symmetric condition in previous works cannot be directly generalized to the case of contrastive losses. 
Secondly, the technical details of deriving the robust condition are different from those of previous works. 



\section{SymNCE: Provably Robust Contrastive Loss}\label{sec::SymNCE}

Theoretically inspired by Theorem \ref{thm::noisetolerant}, we propose a robust contrastive loss SymNCE by directly modifying the InfoNCE loss function. 
In Section \ref{sec::nonrobust}, we first prove that InfoNCE loss is non-robust to label noise. To meet the robust condition, we propose a reverse version of InfoNCE called RevNCE in Section \ref{sec::RevNCE}, and then design the specific form of SymNCE in Section \ref{sec::SymNCE} by adding RevNCE to InfoNCE. The empirical form of SymNCE is shown in Section \ref{sec::empSymNCE}. All proofs are shown in Appendix \ref{sec::proof1}.

\subsection{InfoNCE Loss Function is Non-Robust}\label{sec::nonrobust}

We first argue that the InfoNCE loss function is non-robust.
Specifically, we have the InfoNCE loss function with $M$ positive samples and $K$ negative samples defined as follows.
\begin{align}\label{eq::InfoNCE}
	\mathcal{L}_{\mathrm{InfoNCE}} (\boldsymbol{x}; f) 
	:= \frac{1}{M}\sum_{m \in [M]} \ell(x,x_m^+; f),
\end{align}
where 
\begin{align*}
	\ell(x,x_m^+; f) := -\log\frac{e^{f(x)^\top f(x_m^+)}}{e^{f(x)^\top f(x_m^+)} + \sum_{k \in [K]} e^{f(x)^\top f(x^-_k)}},
\end{align*}

For the convenience of calculation,  we consider the limit form of InfoNCE loss as
\begin{align}
	&\mathcal{R}^{\mathrm{lim}}(\mathcal{L}_{\mathrm{InfoNCE}}; f)
        \nonumber\\
	&= \lim_{M, K \to \infty} \mathcal{R}(\mathcal{L}_{\mathrm{InfoNCE}}; f) - \log K
	\nonumber\\
	&= - \mathbb{E}_{x,x^+ \overset{\text{i.i.d.}}{\sim} \mathrm{P}_X} f(x)^\top f(x^+) 
        \nonumber\\
	&\quad+ \mathbb{E}_{x \sim \mathrm{P}_X} \log \mathbb{E}_{x^- \sim \mathrm{P}_X} e^{f(x)^\top f(x^-)},
\end{align}
inspired by \cite{wang2020understanding}.

According to the definition of the additional risk, we have
{
\begin{align}\label{eq::InfoNCE_lim}
	&\Delta\mathcal{R}^{\mathrm{lim}}(\mathcal{L}_{\mathrm{InfoNCE}}; f)
        \nonumber\\
        &=\lim_{M, K \to \infty} \Delta \mathcal{R}(\mathcal{L}_{\mathrm{InfoNCE}}; f) - \log K 
	\nonumber\\
	&= - \mathbb{E}_{x,x^+ \overset{\text{i.i.d.}}{\sim} \mathrm{P}_X} f(x)^\top f(x^+) 
	\nonumber\\
	&\quad+ \mathbb{E}_{x \sim \mathrm{P}_X} \log \mathbb{E}_{x^- \sim \mathrm{P}_X} e^{f(x)^\top f(x^-)}.
\end{align}
}



Recall that in Theorem \ref{thm::noisetolerant}, the general robust condition requires $\Delta\mathcal{R}(\mathcal{L};f)$ to be a constant, whereas according to the above equation, we observe that the value of $\Delta\mathcal{R}^{\mathrm{lim}}(\mathcal{L}_{\mathrm{InfoNCE}};f)$ is a non-constant depending on the function form of $f$. 
In this case, optimizing the noisy risk $\widetilde{\mathcal{R}}(\mathcal{L}_{\mathrm{InfoNCE}};f)$ is not guaranteed to reach the same solution as the clean risk $\mathcal{R}(\mathcal{L}_{\mathrm{InfoNCE}};f)$, and therefore the InfoNCE loss is non-robust.

\subsection{Reverse InfoNCE (RevNCE)}\label{sec::RevNCE}

Now that we prove that the InfoNCE loss fail to meet the robust condition, we seek to add a reverse term to the InfoNCE loss such whose additional risk equals exactly to $-\Delta \mathcal{R}^{\mathrm{lim}}(\mathcal{L}_{\mathrm{InfoNCE}}; f)$ through all representation functions $f$, so as to guarantee the robustness of the total loss function. 
Although this theoretical robustness can be trivially achieved by adding $-\mathcal{L}_{\mathrm{InfoNCE}}$ or other similar terms to the original InfoNCE, this would destroy the learning procedure.
Therefore, in the meanwhile, this reverse term is also required to function similarly as InfoNCE, i.e. aligning the positive samples and pushing apart the negative ones.

To find a proper ``reverse'' loss that meets the above requirements, we utilize the symmetric property of the additional risk.
Specifically, note that in the definition of $\Delta \mathcal{R}^{\mathrm{lim}} (\mathcal{L}_{\mathrm{InfoNCE}}; f)$, the positive and negative samples are i.i.d.~generated. Therefore, we can exchange them to get a reverse form of $\Delta \mathcal{R}^{\mathrm{lim}} (\mathcal{L}_{\mathrm{InfoNCE}}; f)$ without changing its value, i.e.
	\begin{align}\label{eq::deltaRrev}
            &-\Delta \mathcal{R}^{\mathrm{lim}} (\mathcal{L}_{\mathrm{InfoNCE}}; f) + \log K 
		\nonumber\\
            &=\lim_{M, K \to \infty} -\Delta \mathcal{R} (\mathcal{L}_{\mathrm{InfoNCE}}; f) + \log K 
		\nonumber\\
		&= \mathbb{E}_{x,x^- \overset{\text{i.i.d.}}{\sim} \mathrm{P}_X} f(x)^\top f(x^-) - \mathbb{E}_{x \sim \mathrm{P}_X} \log \mathbb{E}_{x^+ \sim \mathrm{P}_X} e^{f(x)^\top f(x^+)}
		\nonumber\\
		&= \mathbb{E}_{x,x^- \overset{\text{i.i.d.}}{\sim} \mathrm{P}_X} -\log \frac{\mathbb{E}_{x^+ \sim \mathrm{P}_X} e^{f(x)^\top f(x^+)}}{f(x)^\top f(x^-)}.
	\end{align}

In Theorem \ref{prop::RevNCE}, we propose the reverse InfoNCE loss (abbreviated as RevNCE), and show that its limit form is exactly the same as \eqref{eq::deltaRrev}.
Morevoer, similar to InfoNCE, the RevNCE loss aligns the positive samples and meanwhile pushes away the negatives. It differs from the InfoNCE loss function only in the summation of the positives and negatives.

\begin{theorem}\label{prop::RevNCE}
	Define the RevNCE loss function as 
	{
		\begin{align}
			\mathcal{L}_{\mathrm{RevNCE}} (\boldsymbol{x}; f) 
			= \frac{1}{K}\sum_{k \in [K]} -\log \frac{\frac{1}{M}\sum_{m \in [M]} e^{f(x)^\top f(x_m^+)}}{e^{f(x)^\top f(x_k^-)}}.
		\end{align}
	}
	Then we have
 {
	 \begin{align*}
        \Delta \mathcal{R}^{\mathrm{lim}} (\mathcal{L}_{\mathrm{RevNCE}}; f) 
        = - \Delta \mathcal{R}^{\mathrm{lim}} (\mathcal{L}_{\mathrm{InfoNCE}}; f) + \log K,
	 \end{align*}
  where $\Delta \mathcal{R}^{\mathrm{lim}} (\mathcal{L}_{\mathrm{RevNCE}}; f) = \lim_{M, K \to \infty} \Delta \mathcal{R} (\mathcal{L}_{\mathrm{RevNCE}}; f)$
  and $\Delta \mathcal{R}^{\mathrm{lim}} (\mathcal{L}_{\mathrm{InfoNCE}}; f) = \lim_{M, K \to \infty} \Delta \mathcal{R}(\mathcal{L}_{\mathrm{InfoNCE}}; f)$.
  }
\end{theorem}

\begin{proof}[Proof of Theorem \ref{prop::RevNCE}]
	\begin{align}
		&\lim_{M, K \to \infty} \Delta \mathcal{R} (\mathcal{L}_{\mathrm{RevNCE}}; f)
		\nonumber\\
		&= \lim_{K,M \to \infty} \mathbb{E}_{x, x_m^+, x_k^- \overset{\text{i.i.d.}}{\sim} \mathrm{P}_X} 
            \nonumber\\
            &\qquad\qquad\cdot\frac{1}{K}\sum_{k \in [K]} -\log \frac{\frac{1}{M}\sum_{m \in [M]} e^{f(x)^\top f(x_m^+)}}{e^{f(x)^\top f(x_k^-)}}
		\nonumber\\
		&= \frac{1}{K}\sum_{k \in [K]} f(x)^\top f(x_k^-) 
            \nonumber\\
		&- \mathbb{E}_{x, x_m^+, x_k^- \overset{\text{i.i.d.}}{\sim} \mathrm{P}_X} \lim_{K,M \to \infty} \frac{1}{K}\sum_{k \in [K]} \log \Big(\frac{1}{M}\sum_{m \in [M]} e^{f(x)^\top f(x_m^+)}\Big)
		\nonumber\\
		&= \mathbb{E}_{x, x^- \overset{\text{i.i.d.}}{\sim} \mathrm{P}_X} f(x)^\top f(x_k^-) 
            \nonumber\\
		&- \mathbb{E}_{x \sim \mathrm{P}_X} \log \Big(\mathbb{E}_{x^+ \sim \mathrm{P}_X}  e^{f(x)^\top f(x^+)}\Big)
		\nonumber\\
		&= \log K - \lim_{M, K \to \infty} \Delta \mathcal{R}(\mathcal{L}_{\mathrm{InfoNCE}}; f). 
	\end{align}
\end{proof}

\subsection{Symmetric InfoNCE (SymNCE)}\label{sec::alg}

By adding together the InfoNCE and RevNCE loss functions, we propose the symmetric InfoNCE (SymNCE) as 
\begin{align}\label{eq::SymNCE}
	\mathcal{L}_{\mathrm{SymNCE}}(\boldsymbol{x}; f) 
	:= \mathcal{L}_{\mathrm{InfoNCE}}(\boldsymbol{x}; f) + \mathcal{L}_{\mathrm{RevNCE}}(\boldsymbol{x}; f).
\end{align}
By Theorem \ref{prop::RevNCE}, we have 
\begin{align}
&\Delta\mathcal{R}^{\mathrm{lim}}(\mathcal{L}_{\mathrm{SymNCE}}; f) 
\nonumber\\
&=\lim_{M, K \to \infty} \Delta\mathcal{R}(\mathcal{L}_{\mathrm{SymNCE}}; f) 
\nonumber\\
&= \lim_{M, K \to \infty} \Delta\mathcal{R}(\mathcal{L}_{\mathrm{InfoNCE}}; f) 
+ \lim_{M, K \to \infty} \Delta\mathcal{R}(\mathcal{L}_{\mathrm{RevNCE}}; f) 
\nonumber\\
&= \Delta \mathcal{R}^{\mathrm{lim}} (\mathcal{L}_{\mathrm{InfoNCE}}; f)
+ \Delta \mathcal{R}^{\mathrm{lim}} (\mathcal{L}_{\mathrm{RevNCE}}; f) 
\nonumber\\
&= \log K.
\end{align} 
Then by Theorem \ref{thm::noisetolerant}, $\mathcal{L}_{\mathrm{SymNCE}}(\boldsymbol{x}; f)$ is noise tolerant.

Intuitively, $\mathcal{L}_{\mathrm{SymNCE}}$ is noise tolerant because $\mathcal{L}_{\mathrm{RevNCE}}$ selects high-confidence positive samples and increases the positive-negative margin.
Specifically, eecall that for the InfoNCE loss, because log-sum-exp approximates the max function, we have
\begin{align}
    &\mathcal{L}_{\mathrm{InfoNCE}} (\boldsymbol{x}; f)
    \nonumber\\
	 &= \frac{1}{M}\sum_{m \in [M]} 
	\log\bigg(1 + \sum_{k \in [K]} e^{f(x)^\top [f(x_k^-) - f(x_m^+)]}\bigg)
	\nonumber\\
	&\approx \frac{1}{M}\sum_{m \in [M]} \max_{k \in [K]} \big\{0, f(x)^\top [f(x_k^-) - f(x_m^+)]\big\},
\end{align}
that is, given a positive sample $x_m^+$, $\mathcal{L}_{\mathrm{InfoNCE}}$ pushes all negative samples away until they lie further than $x_m^+$, whereas at the same time, it evenly pulls close all positive samples.

On the contrary, for the RevNCE loss, we have
	\begin{align}
		&\mathcal{L}_{\mathrm{RevNCE}} (\boldsymbol{x}; f) - \log M
		\nonumber\\
	    &= \frac{1}{K} \sum_{k \in [K]} -\log\bigg(\sum_{m \in [M]}e^{f(x)^\top[f(x_m^+)-f(x_k^-)]}\bigg)
	\nonumber\\
		&\approx \frac{1}{K} \sum_{k \in [K]} -\max_{m \in [M]} \big\{f(x)^\top[f(x_m^+)-f(x_k^-)]\big\}\nonumber\\
	&= \frac{1}{K} \sum_{k \in [K]} \big\{f(x)^\top f(x_k^-) - \max_{m \in [M]} f(x)^\top f(x_m^+)\big\},
	\end{align}
that is, minimizing $\mathcal{L}_{\mathrm{RevNCE}} (\boldsymbol{x}; f)$ is to optimize over the positive sample sharing the highest similarity with the anchor. 
Note that samples with high semantic similarity usually have the same ground truth label. Then given a negative sample $x_k^-$, $\mathcal{L}_{\mathrm{RevNCE}}$ only pulls near the most confident same-class positive to the anchor point, so $\mathcal{L}_{\mathrm{SymNCE}}$ is robust against label noise because $\mathcal{L}_{\mathrm{RevNCE}}$ adds additional weights to the highly confident same-class positives.
On the other hand, $\mathcal{L}_{\mathrm{RevNCE}}$ intends to push the negative samples as far as possible from the anchor point, instead of just making it further than the positives. This increases the positive-negative margin ensuring a clear discrimination of the most confident same-class positive from the negatives, and consequently further improves robustness against label noise.

Furthermore, it is worthwhile mentioning that the high-level idea of our proposed robust counterpart of the InfoNCE loss is entirely different from previous works. It is worth noting that most previous robust losses have additive forms, i.e. they are summations of either independent positive-pair and negative-pair parts (e.g. RINCE) or independent loss terms over different classes (e.g. SCE). By contrast, the InfoNCE loss itself violates this additive form, so previous ideas of turning InfoNCE into a symmetric robust loss do not directly apply. Besides, it is computationally more burdensome to include $\Delta \mathcal{R}$ than the sum of classification losses over all classes $\sum_{i=1}^C \mathcal{L}(f(x),i)$ in the optimization procedure of loss. Therefore, it is also not a good idea to follow the idea of normalization (e.g. APL) to design a robust contrastive loss. Innovatively, in this paper, we design a robust contrastive loss based on the fact that contrastive losses such as InfoNCE could have multiple positive and negative pairs, which can not be realized by classification losses. 


\subsection{Empirical Form of SymNCE}\label{sec::empSymNCE}

As is argued in previous works, robustness itself is insufficient for empirical performance and robust losses can suffer from underfitting \cite{ma2020normalized}. Therefore, it is common to introduce an additional weight parameter $\beta \in [0,1]$ to balance between accuracy and robustness \cite{chuang2022robust, ma2020normalized, wang2019symmetric}. Empirically, denote $P(i) := \{p \in A(i): \tilde{y}_p = \tilde{y}_i\}$ as the index set of all same-class positives distinct from $i$, and $A(i)$ as the index set of all augmented samples. Then given the weight parameter $\beta$, and the temperature parameter $\tau > 0$, we define the empirical form of SymNCE loss as
\begin{align*}
	\widehat{\mathcal{L}}_{\mathrm{SymNCE}}(x_i; f, \beta) 
	:= \widehat{\mathcal{L}}_{\mathrm{InfoNCE}}(x_i; f) + \beta \cdot  \widehat{\mathcal{L}}_{\mathrm{RevNCE}}(x_i; f),
\end{align*}
where the empirical forms are
{
	\begin{align*}
		\widehat{\mathcal{L}}_{\mathrm{InfoNCE}}(x_i; f)
		:= \frac{1}{\vert P(i)\vert} \sum_{p \in P(i)} -\log \frac{e^{f(x_i)^\top f(x_p)/\tau}}{\sum_{a \in A(i)} e^{f(x_i)^\top f(x_a)/\tau}},
	\end{align*}
}
{
	\begin{align*}
		\widehat{\mathcal{L}}_{\mathrm{RevNCE}}(x_i; f)
		:= \frac{1}{\vert A(i) \vert -1} \sum_{a \in A(i)\setminus \{i\}} \hat{\ell}(x_i,x_p; f),
	\end{align*}
}
and
{
\begin{align*}
	\hat{\ell}(x_i,x_p; f) := -\log \frac{\frac{1}{\vert P(i)\vert} \sum_{p \in P(i)} e^{f(x_i)^\top f(x_p)/\tau}}{e^{f(x_i)^\top f(x_a)/\tau}}.
\end{align*}
}

\section{Robust Condition as An Inclusive Framework}\label{sec::discussion}

Aside from inspiring the robust SymNCE loss based on the InfoNCE loss,
in this part, we highlight that our theoretical analysis in Section \ref{sec::thoerem} in fact serves as an inclusive framework that applies to nearest-neighbor (NN) sample selection, a widely used robust contrastive learning technique, and the robust InfoNCE (RINCE) loss proposed in \cite{chuang2022robust}.

\subsection{NN Selection under the Unified Theoretical Framework}

In the method of NN selection, the positive samples in the supervised contrastive learning are selected as the same-class samples near to the anchor point in the embedding space. 
This technique is often used in noisy label learning algorithms to select confident samples that are believed to have correct annotated labels \cite{navaneet2022constrained, ortego2021multi, li2022selective}. These works discuss that samples usually have the same ground truth label as the semantically similar examples in a neighborhood, so it is reasonable to use NN techniques to select confident samples. 
In this part, we give an alternative theoretical explanation that NN techniques can reduce the gap between the additional risk $\Delta\mathcal{R}$ and constant values, thus making the loss function robust to label noise.


Under our theoretical framework, we take the widely used InfoNCE loss \eqref{eq::InfoNCE} as an example to demonstrate how the NN technique enables contrastive learning to be robust against label noise.

Recall that in Section \ref{sec::nonrobust}, we discuss that InfoNCE is non-robust because $\lim_{M, K \to \infty} \Delta \mathcal{R}(\mathcal{L}_{\mathrm{InfoNCE}}; f) - \log K =  - \mathbb{E}_{x,x^+ \overset{\text{i.i.d.}}{\sim} \mathrm{P}_X} f(x)^\top f(x^+) + \mathbb{E}_{x \sim \mathrm{P}_X} \log \mathbb{E}_{x^- \sim \mathrm{P}_X} e^{f(x)^\top f(x^-)}$ is a non-constant relying on the function form of $f$. By taking a closer look, according to Jensen's Inequality, we see that
$\mathbb{E}_{x'} f(x)^\top f(x') < \log \mathbb{E}_{x'} e^{f(x)^\top f(x')}$ for non-constant-valued $f$, and therefore we have $\lim_{M, K \to \infty} \Delta \mathcal{R}(\mathcal{L}_{\mathrm{InfoNCE}}; f) - \log K > 0$.
Here we discuss that the NN technique reduces the value of the $\mathbb{E}_{x'} f(x)^\top f(x')$, making $\lim_{M, K \to \infty} \Delta \mathcal{R}(\mathcal{L}_{\mathrm{InfoNCE}}; f) - \log K$ close to the constant $0$, and thus enabling the loss function to be noise tolerant according to Theorem \ref{thm::noisetolerant}.
Mathematically, we formulate the InfoNCE loss with the NN technique as 
\begin{align*}
	\mathcal{L}_{\mathrm{InfoNCE-NN}} (\boldsymbol{x}; f, t)
	&:= \frac{1}{\vert\mathcal{I}_N(x;t)\vert}\sum_{m \in \mathcal{I}_N(x;t)} \ell(x,x_m^+;f),
\end{align*}
where for a threshold parameter $t \in [-1,1]$, we let $\mathcal{N}(x;t) := \{x^+: f(x)^\top f(x^+) \geq t\}$ be the neighbor set of sample $x$, and denote $\mathcal{I}_N(x;t)$ as the index set of $\mathcal{N}(x;t)$.
The threshold parameter $t$ ensures that only the positive samples sharing high similarity with the anchor sample are used to calculate the InfoNCE loss. 
Note that when $t$ is selected as the quantiles of $f(x)^\top f(x')$, the positive sample set $\mathcal{N}(x;t)$ contains exactly the nearest neighbors of $x$ in the embedding space. 
Then for the corresponding additional risk w.r.t.~$\mathcal{L}_{\mathrm{InfoNCE-NN}}$, we have
{
\begin{align*}
	&\lim_{M, K \to \infty} \Delta \mathcal{R}(\mathcal{L}_{\mathrm{InfoNCE-NN}}; f, t) - \log K
	\nonumber\\
	&= - \mathbb{E}_{x \sim \mathrm{P}_X} \mathbb{E}_{x' \in \mathcal{N}(x;t)} f(x)^\top f(x') 
	\nonumber\\
	&\quad+ \mathbb{E}_{x \sim \mathrm{P}_X} \log \mathbb{E}_{x' \sim \mathrm{P}_X} e^{f(x)^\top f(x')}.
\end{align*}
}

Because $\mathbb{E}_{x' \in \mathcal{N}(x;t)} f(x)^\top f(x') \geq \mathbb{E}_{x' \sim \mathrm{P}_X} f(x)^\top f(x')$, for a given $f$, we can select a proper threshold parameter $t$ to make $\lim_{M, K \to \infty} \Delta \mathcal{R}(\mathcal{L}_{\mathrm{InfoNCE-NN}}; f, t) - \log K = 0$, and thus $\mathcal{L}_{\mathrm{InfoNCE-NN}}$ is noise tolerant.

\subsection{RINCE under the Unified Theoretical Framework}\label{sec::rince}

The mathematical form of the RINCE loss function \cite{chuang2022robust} is 
	\begin{align*}
		\mathcal{L}_{\mathrm{RINCE}}(\boldsymbol{x}; f) 
		&= -(1-\lambda) e^{f(x)^\top f(x^+)} + \lambda \sum_{k=1}^K e^{f(x)^\top f(x_k^-)},
	\end{align*}
where $\lambda > 0$ is a density weighting term controlling the ratio between positive and negative pairs.
Although inspired by the symmetric condition \cite{ghosh2017robust}, RINCE does not fit this theoretical framework designed for supervised losses that align the model prediction and label. 
By contrast, our proposed inclusive theoretical framework for contrastive losses can guarantee the robustness of RINCE against label noise from the risk consistency perspective.
By definition, we have the additional risk w.r.t.~$\mathcal{L}_{\mathrm{RINCE}}$ as
{
	\begin{align*}
		 &\Delta \mathcal{R}(\mathcal{L}_{\mathrm{RINCE}}; f) 
		 \nonumber\\
		 &= 
		\mathbb{E}_{x, x^+, \{x_k^-\}_{k=1}^K \overset{\text{i.i.d.}}{\sim} \mathrm{P}_X} \Big[-(1-\lambda) e^{f(x)^\top f(x^+)} 
		\nonumber\\
		&\qquad\qquad\qquad\qquad\qquad 
            + \lambda \sum_{k=1}^K e^{f(x)^\top f(x_k^-)}\Big] 
		\nonumber\\
		&= -(1-\lambda) \mathbb{E}_{x, x^+ \overset{\text{i.i.d.}}{\sim} \mathrm{P}_X} e^{f(x)^\top f(x^+)} 
            \nonumber\\
		&\quad
            + \lambda \sum_{k=1}^K \mathbb{E}_{x, x_k^- \overset{\text{i.i.d.}}{\sim} \mathrm{P}_X} e^{f(x)^\top f(x_k^-)} 
		 \nonumber\\
		&= \big((K+1)\lambda - 1\big) \mathbb{E}_{x, x' \overset{\text{i.i.d.}}{\sim} \mathrm{P}_X} e^{f(x)^\top f(x')}.
	\end{align*}
}
According to Theorem \ref{thm::noisetolerant}, when $\lambda = 1/(K+1)$, we have $\Delta \mathcal{R}(\mathcal{L}_{\mathrm{RINCE}}; f) = 0$, and thus $\mathcal{L}_{\mathrm{RINCE}}$ is noise tolerant.
That is, as a byproduct, our theory provides a more specific theoretical choice of the parameter $\lambda$ in RINCE.

\section{Validation Experiments}\label{sec::exp}

This paper primarily focuses on theoretical analysis, explaining how different samples in contrastive learning impact generalization. The experiments in this part are mainly designed to validate the theoretical insights and demonstrate that the proposed directions for improving performance are sound. 
In this section, we evaluate the empirical performance of our proposed SymNCE. 
In Sections \ref{sec::comparisons} and \ref{sec::realdata}, we show the effectiveness of our SymNCE by comparing with state-of-the-art robust loss functions on standard and real-world datasets, respectively. 
In Section \ref{sec::parameter}, we study the effect of the weight parameter $\beta$. 
In Section \ref{sec::verification}, we verify of our theoretical byproduct for the parameter choice in RINCE.

\subsection{Performance Comparisons on Benchmark Datasets}\label{sec::comparisons}

\begin{table*}[!h]
	\centering
	\caption{Performance comparisons with state-of-the-art robust losses on CIFAR datasets.}
		\begin{tabular}{c|c|ccccc|cc}
			\toprule
			\multirow{2}{*}{Dataset}  &
			\multirow{2}{*}{Noise rate} & \multicolumn{5}{c|}{Symmetric} & \multicolumn{2}{c}{Asymmetric} \\
			\cmidrule{3-9}
			& & 0\% & 20\% & 40\% & 60\% & 80\% & 20\% & 40\% \\ 
			\midrule
			\multirow{8}{*}{CIFAR-10} & CE & 92.88 & 85.22 & 78.9 & 71.98 & 41.38 & 86.88 & 82.12 \\
			& MAE & 91.32 & 89.28 & 84.85 & 78.19 & 41.46 & 81.3 & 56.77 \\
			& GCE & 91.83 & 89.22 & 84.66 & 76.66 & 42.21 & 88.01 & 81.05 \\
			& SCE & 92.97 & 89.48 & 83.57 & 77.6 & 55.58 & 89.08 & 82.46 \\
			& APL & 91.21 & 88.9 & 82.08 & 78.48 & 52.04 & 88.39 & 81.92 \\
			\cmidrule{2-9}
			& SupCon & 93.07 & 87.1 & 80.47 & 61.8 & 55.6 & 90.08 & 87.26 \\
			& RINCE & 86.17 & 85.26 & 83.15 & 80.65 & \textbf{80.32} & 84.92 & 84.27 \\
			& SymNCE (Ours) & \textbf{93.12} & \textbf{89.81} & \textbf{85.32} & \textbf{80.89} & 60.74 & \textbf{91.0} & \textbf{88.28} \\
			\midrule
			\multirow{8}{*}{CIFAR-100} & CE & 64.39 & 47.21 & 37.30 & 27.25 & 15.12 & 49.16 & 36.29 \\ 
			& MAE & 13.53 & 8.84 & 8.44 & 6.63 & 2.73 & 11.63 & 7.69 \\ 
			& GCE & 58.52 & 54.16 & 47.27 & 35.65 & 20.25 & 53.79 & 34.60 \\ 
			& SCE & 66.83 & 60.32 & 52.79 & 39.24 & 20.33 & 59.29 & 40.49 \\ 
			& APL & 34.22 & 28.38 & 25.27 & 16.95 & 10.68 & 28.98 & 21.70 \\ 
			\cmidrule{2-9}
			& SupCon & 68.05 & 61.23 & 53.02 & 38.74 & 25.37 & 64.98 & 55.33 \\ 
			& RINCE & 44.41 & 44.29 & 42.27 & 41.46 & 38.99 & 42.49 & 32.68 \\ 
			& SymNCE (Ours) & \textbf{70.30} & \textbf{64.56} & \textbf{55.68}& \textbf{43.35} & \textbf{39.14} & \textbf{67.15} & \textbf{56.5} \\ 
			\bottomrule
		\end{tabular}
	\label{tab::cifar}
\end{table*}

\begin{table*}[!h]
	\centering
	\caption{Real-data comparisons with SOTA robust losses.}
	\begin{tabular}{ccccc|ccc}
		\toprule
		CE & GCE & SCE & APL & MAE & SupCon & RINCE & SymNCE \\ 
		\midrule
		53.48 & 56.01 & 57.48 & 36.15 & 36.83 & 62.50 & 43.30 & \textbf{65.20} \\ 
		\bottomrule
	\end{tabular}
	\label{tab::real}
\end{table*}

\begin{figure*}[!h]
	\centering
	\subfloat[\footnotesize $\beta$ under symmetric noise.]{
		\includegraphics[width=0.31\linewidth, trim = 0 10 0 0]{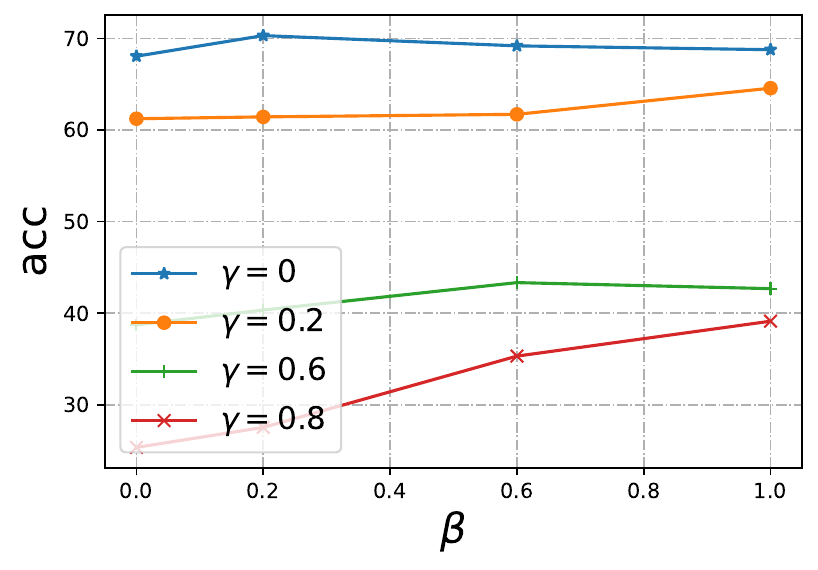}
		\label{fig::clnl_sym}
	}
	\subfloat[\footnotesize $\beta$ under asymmetric noise.]{
		\includegraphics[width=0.31\linewidth, trim = 0 10 0 0]{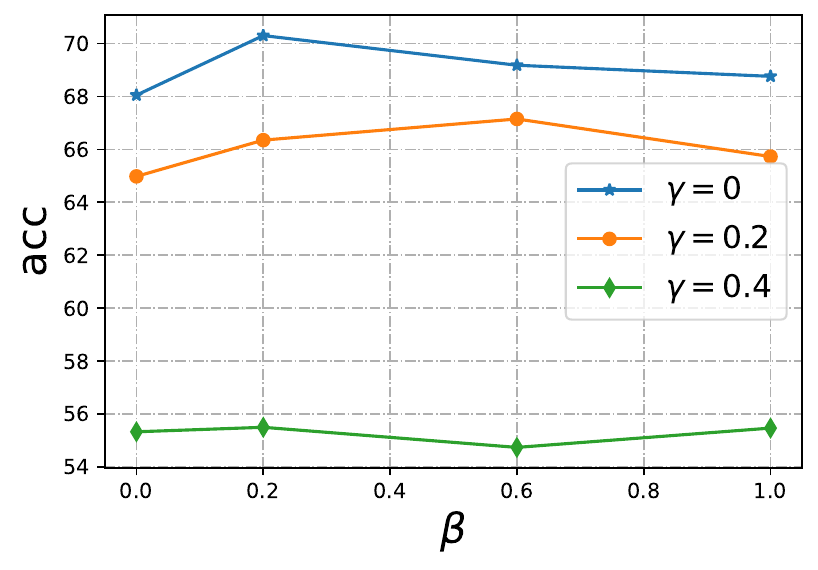}
		\label{fig::clnl_asym}
	}
        \subfloat[\footnotesize $\lambda$ in RINCE.]{
		\includegraphics[width=0.31\linewidth, trim = 10 0 40 0]{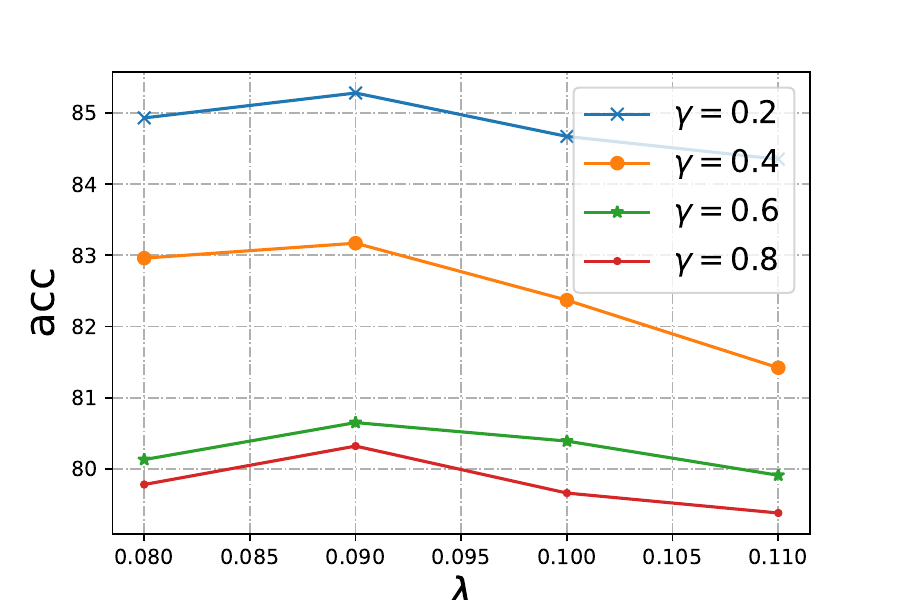}
		\label{fig::rince_10}
	}
	\caption{(a)(b) Parameter analysis of weight parameter $\beta$ in SymNCE under symmetric and asymmetric label noise. (c) Verification of theoretical choice of $\lambda$ in RINCE.}
	\label{fig::weight}
\end{figure*}

\begin{table}[!h]
	\centering
	\caption{Performance comparisons with state-of-the-art robust losses on TinyImagenet.}
		\begin{tabular}{c|ccccc}
			\toprule
                Noise rate
			& 0\% & 20\% & 40\% & 60\% & 80\% \\
			\midrule
			CE & 52.04 & 36.72 & 26.76 & 19.14 & 6.97 \\ 
			MAE & 27.04 & 12.60 & 3.97 & 2.12 & 1.25 \\
			GCE & 47.01 & 35.47 & 28.22 & 22.18 & 11.02 \\ 
			SCE & 48.13 & 39.05 & 33.62 & 16.95 & 10.81 \\ 
			APL & 32.36 & 20.28 & 6.97 & 2.38 & 2.08 \\ 
			\midrule
			SupCon & 56.24 & 48.91 & 44.05 & 40.27 & 28.19 \\ 
			RINCE & 33.80 & 32.82 & 30.67 & 27.99 & 26.95 \\ 
			SymNCE & \textbf{56.33} & \textbf{51.26} & \textbf{46.06}& \textbf{40.83} & \textbf{30.62} \\ 
			\bottomrule
		\end{tabular}
	\label{tab::imagenet}
\end{table}

We first conduct numerical comparisons on CIFAR-10, CIFAR-100, and TinyImagenet benchmark datasets \cite{krizhevsky2009learning}. The noisy labels are generated following standard approaches in previous works \cite{ma2020normalized, wang2019symmetric}. The symmetric label noise is generated by flipping a proportion of labels in each class uniformly at random to other classes. The proportion of flipped labels equals to the noise rate $\gamma$. For asymmetric noise, we flip the labels within a specific set of classes. 
For CIFAR-10, flipping TRUCK $\rightarrow$ AUTOMOBILE, BIRD $\rightarrow$ AIRPLANE, DEER $\rightarrow$ HORSE, and CAT $\leftrightarrow$ DOG.
For CIFAR-100, the 100 classes are grouped into 20 super-classes with each having 5 sub-classes. We then flip each class within the same super-class into the next in a circular fashion. 
We vary the noise rate $\gamma \in \{0.2, 0.4, 0.6, 0.8\}$ for symmetric label noise and $\gamma \in \{0.2, 0.4\}$ for the asymmetric case.

It is worth mentioning that our paper focuses on designing new robust loss functions with sound theoretical guarantees, which is one of the most important parts of learning from label noise. We believe it is not so fair to directly compare a robust loss function with such algorithms incorporating multiple heuristic techniques and without theoretical guarantees. Therefore, we compare our proposed SymNCE with both robust supervised classification losses and robust supervised contrastive losses. For supervised classification losses, we compare with Cross Entropy (CE), Mean Absolute Error (MAE) \cite{ghosh2017robust}, Generalized Cross Entropy (GCE) \cite{zhang2018generalized}, Symmetric Cross Entropy (SCE) \cite{wang2019symmetric}, and Active Passive Loss (APL).
For supervised contrastive losses, we compare with SupCon \cite{khosla2020supervised} and Robust InfoNCE (RINCE) \cite{chuang2022robust}.

We run all experiments on an NVIDIA GeForce RTX 3090 24G GPU. We adopt a ResNet-18 as the backbone for all methods, and use the SGD optimizer with momentum 0.9. The batch size is set to be 512. For our SymNCE, we set the learning rate as 0.1 for CIFAR-10, 0.01 for CIFAR100, and 0.01 for Tinyimagenet without decay. Temperature is 0.5 for CIFAR-10, 0.07 for CIFAR-100, and 0.1 for Tinyimagenet. Weight decay is set to be $10^-4$. The weight parameter $\beta$ is selected in \{0, 0.2, 0.6, 1.0\} through validation. For a robust start, we warm up with the SupCon loss in the early stage of training. For all supervised contrastive losses, we run 300 epochs for training the representations are then evaluated through linear probing on the noisy dataset with CE loss for 30 epochs. The data augmentations are random crop and resize (with random flip), color distortion, and color dropping. For supervised classification losses, we train all losses for 300 epochs and report the test accuracy. 

For compared methods, the parameters are set referring to their original papers. For SupCon, we set the learning rate as 0.5 for CIFAR-10, 0.01 for CIFAR-100, and Tinyimagenet without decay. Temperature is 0.1 for CIFAR-10, 0.07 for CIFAR-100, and 0.1 for Tinyimagenet. Weight decay is set to be $10^-4$. For RINCE, we set temperature 0.1, $\lambda=0.01$ and $q=0.5$. For RINCE, SCE, AP, CE, and MAE, the learning rate is set to be 0.01 without decay. We set $q=0.7$ in GCE, $\alpha =0.1$, $\beta=1$ for SCE. For APL (NCE+MAE), we set $\alpha=\beta=1$ for CIFAR-10 and $\alpha=10$, $\beta=0.1$ for CIFAR-100 and Tinyimagenet.

The results are shown in Tables \ref{tab::cifar} and \ref{tab::imagenet}. We observe that our SymNCE shows mostly better empirical performances under both symmetric and asymmetric label noise across various noise rates.
Compared with classification losses, SymNCE has significant advantages under high noise rates and asymmetric label noise. 
Compared with both classification and contrastive losses, SymNCE has larger performance gains on the more complex CIFAR-100 dataset.

\subsection{Performance Comparisons on Real-world Datasets}\label{sec::realdata}

Besides benchmark datasets, we also evaluate our proposed SymNCE on the real-world dataset Clothing1M \cite{xiao2015learning}, whose noise rate is estimated to be around 40\%. Following \cite{li2020dividemix}, we select a subset containing 1.4k samples for each class from the noisy training data and report the performance on the 10k test data. 
We adopt an ImageNet pre-trained ResNet-18 as our backbone. We use the SGD optimizer with momentum 0.9. The batch size is set to be 32. For our method, we set the learning rate as 0.002. The temperature is 0.07 and weight decay is set to be $10^{-4}$. For robust start, we warm up with the  SupCon loss for the early 20 epochs. For all supervised contrastive losses, we run 80 epochs for representation learning, which are then evaluated through linear probing on the training dataset with CE loss for 30 epochs. For the compared method, we all set the training epochs as 80 for a fair comparison. 

In Table \ref{tab::real}, where the best performance is marked in \textbf{bold}, we show that our SymNCE outperforms both robust classification losses and robust contrastive losses on the real-world dataset Clothing1M.

\subsection{Parameter Analysis}\label{sec::parameter}

Recall that in Section \ref{sec::alg}, we introduce a weight parameter $\beta$ to our $\mathcal{L}_{\mathrm{SymNCE}}$, which is the only parameter of our method and balances between the accurate term $\mathcal{L}_{\mathrm{InfoNCE}}$ and the robust term $\mathcal{L}_{\mathrm{RevNCE}}$. 
We here conduct analysis of the weight parameter $\beta$.
In Figures \ref{fig::clnl_sym} and \ref{fig::clnl_asym}, we plot the test accuracy of SymNCE with different weight parameters $\beta \in \{0.2, 0.6, 1.0\}$ under symmetric label noise $\gamma \in \{0, 0.2, 0.4, 0.6, 0.8\}$ and asymmetric label noise $\gamma \in \{0, 0.2, 0.4\}$ on CIFAR-100, where $\gamma=0$ corresponds to the clean dataset without label noise. 
We show that the optimal $\beta$ increases as noise rate $\gamma$ enhances. 
Specifically, for symmetric label noise, when noise rate is low ($\gamma = 0$), the optimal $\beta=0.2$, and when noise rate is high ($\gamma \in \{0.2, 0.6, 0.8\}$), the optimal $\beta=1.0$. For asymmetric label noise, the optimal $\beta$ is $0.2$ when $\gamma=0$, whereas raises to $1.0$ when $\gamma=0.4$.
This is because robust loss functions are designed to avoid overfitting to label noise, and thus can suffer from underfitting \cite{ma2020normalized} when noise rate is low. Therefore, we require relatively low $\beta$ and focus more on the accurate term $\mathcal{L}_{\mathrm{InfoNCE}}$ when noise rate is low. On the contrary, when noise rate is high, we require $\beta=1$, which is theoretically proved to be robust in Section \ref{sec::alg}.


\subsection{Verification of Our Theoretical Byproduct for RINCE}\label{sec::verification}

Recall that in Section \ref{sec::rince}, when unifying RINCE into our theoretical framework, we could byproduct provide a theoretical optimal parameter $\lambda=1/(K+1)$ for RINCE. Here we conduct experiments to verify the claim. We run experiments with ResNet-18 on the CIFAR-10 dataset under symmetric label noise $\lambda \in \{20\%, 40\%, 60\%, 80\%\}$. We vary the parameter $\lambda=\{0.08,0.09,0.10, 0.11\}$, and illustrate the linear probing accuracy in Figure \ref{fig::rince_10}. We show that RINCE with $\lambda=0.09$ has the best accuracy, which coincides with our theoretical choice $\lambda=1/(K+1)=1/11\approx 0.09$.



\section{Conclusion}

In this paper, we proposed a unified theoretical framework for robust supervised contrastive losses against label noise. We derived a general robust condition for arbitrary contrastive losses, which further inspires us to propose the SymNCE loss, a direct robust counterpart of the widely used InfoNCE loss. As a theoretical work, our results are naturally limited by the assumptions. Nonetheless, we highlight that our theoretical analysis is a unified framework applying to multiple robust techniques such as nearest neighbor sample selection and robust losses. This framework not only provides a benchmark for assessing the resilience of contrastive losses but also holds promise as a springboard for innovative loss function designs in the future.




\bibliography{NLCL}
\bibliographystyle{IEEEtran}

\clearpage
\appendix
In this appendix, we present the proofs related to Sections \ref{sec::robust_theory} and \ref{sec::SymNCE} in Appendices \ref{sec::proof1} and \ref{sec::proof2} respectively. The additional theoretical results and proofs under asymmetric label noise are in Appendix \ref{sec::asym}.

\section{Appendices}
\subsection{Proofs Related to Section \ref{sec::robust_theory}}\label{sec::proof1}

\begin{proof}[Proof Lemma \ref{lem::decom}]
	By definition of the noisy risk in \eqref{eq::noisyrisk}, there holds
 {
	\begin{align}\label{eq::tildeR1}
		&\widetilde{\mathcal{R}}(\mathcal{L}; f) 
            \nonumber\\
		&= \mathbb{E}_{\tilde{c}^+, \{\tilde{c}_k^-\}_{k=1}^K \sim \tilde{\boldsymbol{\pi}}^{K+1}} \mathbb{E}_{\substack{x, x^+ \sim \tilde{\rho}_{\tilde{c}^+}\\
				x^-_k \sim \tilde{\rho}_{\tilde{c}_k^-}, k \in [K]}} 
		\mathcal{L}(x, x^+, \{x_k^-\}_{k=1}^K; f)
		\nonumber\\
		&= \sum_{i, j_1, \ldots, j_K \in [C]} \tilde{\pi}_i \tilde{\pi}_{j_1} \cdots \tilde{\pi}_{j_K} 
            \nonumber\\
            &\qquad\cdot
		\mathbb{E}_{\substack{x, x^+ \sim \tilde{\rho}_{i}\\
				x^-_k \sim \tilde{\rho}_{j_k}, k \in [K]}} 
		\mathcal{L}(x, x^+, \{x_k^-\}_{k=1}^K; f)
		\nonumber \\
		&= \sum_{i, j_1, \ldots, j_K \in [C]} \sum_{m \in [C]} q_m(i) \pi_m 
            \nonumber\\
            &\qquad\cdot\sum_{l_1 \in [C]} q_{l_1}(j_1) \pi_{l_1} \cdots \sum_{l_K \in [C]} q_{l_K}(j_K) \pi_{l_K}
		\nonumber\\
		&\qquad\cdot
		\mathbb{E}_{\substack{x, x^+ \sim \tilde{\rho}_{i}\\
				x^-_k \sim \tilde{\rho}_{j_k}, k \in [K]}} 
		\mathcal{L}(x, x^+, \{x_k^-\}_{k=1}^K; f).
	\end{align}
 }
	Denote $a(i) := \sum_{u\in[C]} q_u(i) \pi_u$.
	Then by Lemma \ref{lem::pi}, we have
 {
	\begin{align}\label{eq::tildeR2}
		&\mathbb{E}_{\substack{x, x^+ \sim \tilde{\rho}_{i}\\
				x^-_k \sim \tilde{\rho}_{j_k}, k \in [K]}} 
		\mathcal{L}(x, x^+, \{x_k^-\}_{k=1}^K; f)
		\nonumber\\
		&= \idotsint \tilde{\rho}_i(x) \tilde{\rho}_i(x^+) \tilde{\rho}_{j_1}(x_1^-) \cdots \tilde{\rho}_{j_K}(x_K^-) 
            \nonumber\\
            &\ \cdot \mathcal{L}(x, x^+, \{x_k^-\}_{k=1}^K; f) \, dx \, dx^+ \, dx^-_1 \ldots dx^-_K
		\nonumber\\
		&= \idotsint \frac{\sum_{u\in[C]} q_u(i) \rho_u(x) \pi_u}{\sum_{u\in[C]} q_u(i) \pi_u} \cdot \frac{\sum_{u^+\in[C]} q_{u^+}(i) \rho_{u^+}(x^+) \pi_{u^+}}{\sum_{u^+\in[C]} q_{u^+}(i) \pi_{u^+}} 
            \nonumber\\
            &\ \cdot \frac{\sum_{v_1\in[C]} q_{v_1}(j_1) \rho_{v_1}(x_1^-) \pi_{v_1}}{\sum_{v_1\in[C]} q_{v_1}(j_1) \pi_{v_1}} 
		\cdots \frac{\sum_{v_K\in[C]} q_{v_K}(j_K) \rho_{v_K}(x_K^-) \pi_{v_K}}{\sum_{v_K\in[C]} q_{v_K}(j_K) \pi_{v_K}} 
            \nonumber\\
		&\ \cdot \mathcal{L}(x, x^+, \{x_k^-\}_{k=1}^K; f) \, dx \, dx^+ \, dx^-_1 \ldots dx^-_K
		\nonumber\\
		&:= (a(i)^2 a(j_1) \ldots a(j_K))^{-1} 
            \nonumber\\
            &
            \cdot q_u(i) q_{u^+}(i) q_{v_1}(j_1) \cdots q_{v_K}(j_K)
		\nonumber\\
		&\ \cdot \idotsint \rho_u(x) \rho_{u^+}(x^+) \rho_{v_1}(x_1^-) \cdots \rho_{v_K}(x_K^-) 
            \nonumber\\
            &\ \cdot \mathcal{L}(x, x^+, \{x_k^-\}_{k=1}^K; f) \, dx \, dx^+ \, dx^-_1 \ldots dx^-_K
		\nonumber\\
		&:= (a(i)^2 a(j_1) \ldots a(j_K))^{-1} 
            \nonumber\\
            & \hspace{-3mm} \sum_{u, u^+, v_1, \ldots, v_K \in[C]} \pi_u \pi_{u^+} \pi_{v_1} \cdots \pi_{v_K} \cdot q_u(i) q_{u^+}(i) q_{v_1}(j_1) \cdots q_{v_K}(j_K)
		\nonumber\\
		&\ \cdot \mathbb{E}_{\substack{x \sim \rho_u, x^+ \sim \rho_{u^+}\\
				x^-_k \sim \rho_{v_k}, k \in [K]}} 
		\mathcal{L}(x, x^+, \{x_k^-\}_{k=1}^K; f)
		\nonumber\\
		&:= (a(i)^2 a(j_1) \ldots a(j_K))^{-1} 
		\nonumber\\
		&\cdot \sum_{u, u^+, v_1, \ldots, v_K \in[C]} \pi_u \pi_{u^+} \pi_{v_1} \cdots \pi_{v_K} 
		\nonumber\\ 	
            &\ \cdot q_u(i) q_{u^+}(i) q_{v_1}(j_1) \cdots q_{v_K}(j_K) A(u,u^+,\boldsymbol{v}),
	\end{align}
    where we denote 
    \begin{align}
        A(u,u^+,\boldsymbol{v}) := \mathbb{E}_{\substack{x \sim \rho_u, x^+ \sim \rho_{u^+}\\
				x^-_k \sim \rho_{v_k}, k \in [K]}} 
		\mathcal{L}(x, x^+, \{x_k^-\}_{k=1}^K; f).
    \end{align}
 }
	Combining \eqref{eq::tildeR1} and \eqref{eq::tildeR2}, we have
 {
	\begin{align}
		&\widetilde{\mathcal{R}}(\mathcal{L}; f) 
            \nonumber\\
		&= \sum_{i, j_1, \ldots, j_K \in [C]} (a(i)^2 a(j_1) \ldots a(j_K))^{-1} 
		\nonumber\\
		&\qquad\cdot \sum_{m, l_1, \ldots, l_K \in [C]} q_m(i) q_{l_1}(j_1) \cdots q_{l_K}(j_K) \cdot \pi_m \pi_{l_1} \cdots \pi_{l_K} 
		\nonumber\\
		&\qquad\cdot \sum_{u, u^+, v_1, \ldots, v_K \in[C]} q_u(i) q_{u^+}(i) q_{v_1}(j_1) \cdots q_{v_K}(j_K) 
		\nonumber\\ 
		&\qquad\cdot \pi_u \pi_{u^+} \pi_{v_1} \cdots \pi_{v_K} A(u,u^+,\boldsymbol{v})
		\nonumber\\
		&= \sum_{i \in [C]}
		a(i)^{-2} \sum_{m \in [C]} \sum_{u\in[C]} \sum_{u^+\in[C]} \pi_m \pi_u \pi_{u^+} q_m(i) q_u(i) q_{u^+}(i) 
		\nonumber\\
		&\qquad\cdot \sum_{j_1 \in [C]} a(j_1)^{-1} \sum_{l_1 \in [C]} \sum_{v_1\in[C]} \pi_{l_1} \pi_{v_1} q_{l_1}(j_1) q_{v_1}(j_1) \cdots 
		\nonumber\\
		&\qquad\cdot \sum_{j_K \in [C]}
		a(j_K)^{-1} \sum_{l_K \in [C]} \sum_{v_K\in[C]} \pi_{l_K} \pi_{v_K} q_{l_K}(j_K) q_{v_K}(j_K)  
            \nonumber\\
		&\qquad\cdot A(u,u^+,\boldsymbol{v}).
		\label{eq::Anegs}
	\end{align}
 }
	For the positive term in \eqref{eq::Anegs}, we have
        {
	\begin{align}\label{eq::Apos}
		&\sum_{i \in [C]}
		a(i)^{-2} \sum_{m \in [C]} \sum_{u\in[C]} \sum_{u^+\in[C]} \pi_m \pi_u \pi_{u^+} q_m(i) q_u(i) q_{u^+}(i) 
            \nonumber\\
            &\qquad\qquad\qquad\qquad\qquad\qquad \cdot A(u,u^+,\boldsymbol{v})
		\nonumber\\
		&= \sum_{i \in [C]}
		a(i)^{-2} \sum_{u\in[C]} \sum_{u^+\in[C]} \pi_u \pi_{u^+} q_u(i) q_{u^+}(i) A(u,u^+,\boldsymbol{v}) a(i)
		\nonumber\\
		&= \sum_{i \in [C]}
		a(i)^{-1} \sum_{u\in[C]} \sum_{u^+\in[C]} \pi_u \pi_{u^+} q_u(i) q_{u^+}(i) A(u,u^+,\boldsymbol{v}).
	\end{align}
        }
	For the negative terms in \eqref{eq::Anegs}, we have for $k \in [K]$
        {
	\begin{align}\label{eq::Anegk}
		&\sum_{j_k \in [C]} a(j_k)^{-1} \sum_{l_k \in [C]} 	\sum_{v_k\in[C]} \pi_{l_k} \pi_{v_k} q_{l_k}(j_k) q_{v_k}(j_k) A(u,u^+,\boldsymbol{v})
		\nonumber\\
		&= \sum_{j_k \in [C]} a(j_k)^{-1} \sum_{l_k \in [C]} 	q_{l_k}(j_k) \pi_{l_k} \sum_{v_k\in[C]} \pi_{v_k}  q_{v_k}(j_k) A(u,u^+,\boldsymbol{v})
		\nonumber\\
		&= \sum_{j_k \in [C]} \sum_{v_k\in[C]} \pi_{v_k}  	q_{v_k}(j_k) A(u,u^+,\boldsymbol{v})
		\nonumber\\
		&= \sum_{v_k\in[C]} \pi_{v_k} A(u,u^+,\boldsymbol{v}) 	\sum_{j_k \in [C]} q_{v_k}(j_k) 
		\nonumber\\
		&= \sum_{v_k\in[C]} \pi_{v_k} A(u,u^+,\boldsymbol{v}).
	\end{align}
 }
	Then combining \eqref{eq::Anegs}, \eqref{eq::Apos}, and \eqref{eq::Anegs}, we have
        {
	\begin{align}
		&\widetilde{\mathcal{R}}(\mathcal{L}; f) 
		\nonumber\\
            &= \sum_{i \in [C]} 
		a(i)^{-1} \sum_{u\in[C]} \sum_{u^+\in[C]} \pi_u \pi_{u^+} q_u(i) q_{u^+}(i) 
		\nonumber\\
		&\qquad\cdot\sum_{v_1, \ldots, v_K \in[C]} \pi_{v_1} \cdots \pi_{v_K} A(u,u^+,\boldsymbol{v})
		\nonumber\\
		&=\sum_{i \in [C]} 
		a(i)^{-1} \sum_{u\in[C]} \sum_{u^+\in[C]} \pi_u \pi_{u^+} 	q_u(i) q_{u^+}(i) 
		\nonumber\\
		&\qquad\cdot\mathbb{E}_{\substack{x \sim \rho_u,\\x^+ \sim \rho_{u^+}}}
		\mathbb{E}_{\{x_k^-\}_{k=1}^K \overset{\text{i.i.d.}}{\sim} \mathrm{P}_X} \mathcal{L}(\boldsymbol{x}; f).
	\end{align}
	}
\end{proof}

\begin{proof}[Proof of Theorem \ref{thm::riskconsistency}]
	Under Assumption \ref{ass::symnoise}, $q_i(i)=1-\gamma$ and $q_u(i)=\gamma/(C-1)$ for $u \neq i$. Denote $B(u,u^+) := \sum_{v_1, \ldots, v_K \in[C]} \pi_{v_1} \cdots \pi_{v_K} A(u,u^+,\boldsymbol{v})$. By Lemma \ref{lem::decom}, we have
        {
	\begin{align}
		&\widetilde{\mathcal{R}}(\mathcal{L}; f) 
		\nonumber\\
		&= \sum_{i \in [C]} 
		a(i)^{-1} \sum_{u\in[C]} \sum_{u^+\in[C]} \pi_u \pi_{u^+} q_u(i) q_{u^+}(i) B(u,u^+)
		\nonumber\\
		&= \sum_{i \in [C]} a(i)^{-1} (1-\gamma)^2 \pi_i^2 B(i,i)
		\nonumber\\
		&+ \sum_{i \in [C]} a(i)^{-1} \frac{\gamma(1-\gamma)}{C-1} \pi_i \sum_{u^+ \neq i} \pi_{u^+} B(i,u^+)
		\nonumber\\
		&+ \sum_{i \in [C]} a(i)^{-1} \frac{\gamma(1-\gamma)}{C-1} \pi_i \sum_{u \neq i} \pi_{u} B(u,i)
		\nonumber\\
		&+ \sum_{i \in [C]} a(i)^{-1} \frac{\gamma^2}{(C-1)^2} \sum_{u \neq i} \sum_{u^+ \neq i} \pi_u \pi_{u^+} B(u,u^+)
		\nonumber\\
		&= \sum_{i \in [C]} a(i)^{-1} \Big((1-\gamma)^2 - \frac{\gamma^2}{(C-1)^2}\Big) \pi_i^2 B(i,i) 
		\nonumber\\
		&+ \sum_{i \in [C]} a(i)^{-1} \Big(\frac{\gamma(1-\gamma)}{C-1} - \frac{\gamma^2}{(C-1)^2}\Big)\pi_i \sum_{u^+ \neq i} \pi_{u^+} B(i,u^+)
		\nonumber\\
		&+ \sum_{i \in [C]} a(i)^{-1} \Big(\frac{\gamma(1-\gamma)}{C-1} - \frac{\gamma^2}{(C-1)^2}\Big)\pi_i \sum_{u \neq i} \pi_{u} B(u, i)
		\nonumber\\
		&+ \sum_{i \in [C]} a(i)^{-1} \frac{\gamma^2}{(C-1)^2} \sum_{u \in [C]} \sum_{u^+ \in [C]} \pi_u \pi_{u^+} B(u,u^+)
		\nonumber\\
		&= \sum_{u \in [C]} a(u)^{-1} \Big(1-\frac{C}{C-1}\gamma\Big)^2 \pi_u^2 B(u,u) 
		\nonumber\\
		&+ \sum_{u \in [C]} a(u)^{-1} \frac{\gamma}{C-1}\Big(1-\frac{C}{C-1}\gamma\Big)\pi_u \sum_{u^+ \in [C]} \pi_{u^+} B(u,u^+)
		\nonumber\\
		&+ \sum_{u^+ \in [C]} a(u^+)^{-1} \frac{\gamma}{C-1}\Big(1-\frac{C}{C-1}\gamma\Big)\pi_{u^+} \sum_{u \in [C]} \pi_{u} B(u, u^+)
		\nonumber\\
		&+ \sum_{i \in [C]} a(i)^{-1} \frac{\gamma^2}{(C-1)^2} \sum_{u \in [C]} \sum_{u^+ \in [C]} \pi_u \pi_{u^+} B(u,u^+).
	\end{align}
 }
	When the input data is class balanced, i.e. $\pi_i=\frac{1}{C}$ for $i \in [C]$, we have for $i \in [C]$,
		$a(i) = \sum_{u\in[C]}q_u(i)\pi_u = \frac{1}{C} \sum_{u\in[C]}q_u(i) = \frac{1}{C}$.
	Then we have
        {
	\begin{align}
		&\widetilde{\mathcal{R}}(\mathcal{L}; f) 
		\nonumber\\
		&= \Big(1-\frac{C}{C-1}\gamma\Big)^2 \sum_{u \in [C]} \pi_u B(u,u) 
		\nonumber\\
		&+ \frac{C\gamma}{C-1}\Big(1-\frac{C}{C-1}\gamma\Big) \sum_{u \in [C]} \pi_u \sum_{u^+ \in [C]} \pi_{u^+} B(u,u^+)
		\nonumber\\
		&+ \frac{C\gamma}{C-1}\Big(1-\frac{C}{C-1}\gamma\Big) \sum_{u^+ \in [C]} \pi_{u^+} \sum_{u \in [C]} \pi_{u} B(u, u^+)
		\nonumber\\
		&+ \frac{C^2\gamma^2}{(C-1)^2} \sum_{u \in [C]} \sum_{u^+ \in [C]} \pi_u \pi_{u^+} B(u,u^+)
		\nonumber\\
		&= \Big(1-\frac{C}{C-1}\gamma\Big)^2 \sum_{u \in [C]} \pi_u B(u,u) 
		\nonumber\\
		&+ \frac{C\gamma}{C-1}\Big(2-\frac{C}{C-1}\gamma\Big) \sum_{u,u^+ \in [C]} \pi_u \pi_{u^+} B(u,u^+)
	\end{align}
 }
	Note that 
        {
	\begin{align}
		&\sum_{u \in [C]} \pi_u B(u,u) 
            \nonumber\\
		&= \sum_{u \in [C]} \pi_u \sum_{v_1, \ldots, v_K \in[C]} \pi_{v_1} \cdots \pi_{v_K} 
        \nonumber\\
        &\qquad\cdot\mathbb{E}_{\substack{x \sim \rho_u, x^+ \sim \rho_{u}\\
				x^-_k \sim \rho_{v_k}, k \in [K]}} 
		\mathcal{L}
  (x, x^+, \{x_k^-\}_{k=1}^K; f)
		\nonumber\\
		&= \mathbb{E}_{c^+, \{c_k^-\}_{k=1}^K \sim \boldsymbol{\pi}^{K+1}} \mathbb{E}_{\substack{x \sim \rho_u, x^+ \sim \rho_{u}\\
				x^-_k \sim \rho_{v_k}, k \in [K]}} 
		\mathcal{L}
  (x, x^+, \{x_k^-\}_{k=1}^K; f)
		\nonumber\\
		&= \mathcal{R}(\mathcal{L}; f),
	\end{align}
        }
	and that 
        {
	\begin{align}
		&\sum_{u,u^+ \in [C]} \pi_u \pi_{u^+} B(u,u^+)
		\nonumber\\
		&= \sum_{u,u^+ \in [C]} \pi_u \pi_{u^+} \sum_{v_1, \ldots, v_K \in[C]} \pi_{v_1} \cdots \pi_{v_K} 
            \nonumber\\
            &\qquad\cdot\mathbb{E}_{\substack{x \sim \rho_u, x^+ \sim \rho_{u^+}\\
				x^-_k \sim \rho_{v_k}, k \in [K]}} 
		\mathcal{L}
  (x, x^+, \{x_k^-\}_{k=1}^K; f)
		\nonumber\\
		&= \mathbb{E}_{c,\ c^+, \{c_k^-\}_{k=1}^K \sim \boldsymbol{\pi}^{K+2}} \mathbb{E}_{\substack{x \sim \rho_u, x^+ \sim \rho_{u^+}\\
				x^-_k \sim \rho_{v_k}, k \in [K]}} 
		\mathcal{L}
  (x, x^+, \{x_k^-\}_{k=1}^K; f)
		\nonumber\\
		&:= \Delta \mathcal{R}(\mathcal{L}; f).
	\end{align}
        }
	Thus we have
        {
	\begin{align*}
		\widetilde{\mathcal{R}}(\mathcal{L}; f) 
		&= \Big(1-\frac{C}{C-1}\gamma\Big)^2 \mathcal{R}(\mathcal{L}; f) 
		\nonumber\\
		&+ \frac{C\gamma}{C-1}\Big(2-\frac{C}{C-1}\gamma\Big) \Delta \mathcal{R}(\mathcal{L}; f).
	\end{align*}
        }
\end{proof}

\begin{proof}[Proof of Theorem \ref{thm::noisetolerant}]
	For symmetric label noise, by Theorem \ref{thm::riskconsistency} and that $\Delta \mathcal{R}(\mathcal{L}; f) = A$, we have
        {
	\begin{align*}
		\widetilde{\mathcal{R}}(\mathcal{L}; f) 
		&= \Big(1-\frac{C}{C-1}\gamma\Big)^2 \mathcal{R}(\mathcal{L}; f) 
		+ \frac{C\gamma}{C-1}\Big(2-\frac{C}{C-1}\gamma\Big) A.
	\end{align*}
        }
	Suppose $f^*$ is the minimizer of $\widetilde{\mathcal{R}}(\mathcal{L}; f)$, i.e. for all $f$
	\begin{align}
		\widetilde{\mathcal{R}}(\mathcal{L}; f^*) \leq \widetilde{\mathcal{R}}(\mathcal{L}; f).
	\end{align}
	Then if $\gamma \leq (C-1)/C$, we have
	\begin{align}
		\mathcal{R}(\mathcal{L}; f^*)
		\leq \mathcal{R}(\mathcal{L}; f),
	\end{align}
	that is, $f^*$ is also the minimizer of $\mathcal{R}(\mathcal{L}; f)$.
\end{proof}

\subsection{Proofs Related to Section \ref{sec::SymNCE}}\label{sec::proof2}

For completeness, we first prove the limit form of InfoNCE shown in \eqref{eq::InfoNCE_lim}, following \cite{wang2020understanding}.

\begin{align}
	&\lim_{M, K \to \infty} \Delta \mathcal{R}(\mathcal{L}_{\mathrm{InfoNCE}}; f) - \log K 
	\nonumber\\
	&=  - \log K + \lim_{M, K \to \infty} \mathbb{E}_{x, x_m^+, x_k^- \overset{\text{i.i.d.}}{\sim} \mathrm{P}_X} 
        \nonumber\\
        &\qquad\quad\cdot\frac{1}{M}\sum_{m \in [M]} -\log\frac{e^{f(x)^\top f(x_m^+)}}{e^{f(x)^\top f(x_m^+)} + \sum_{k \in [K]} e^{f(x)^\top f(x^-_k)}} 
	\nonumber\\
	&= \mathbb{E}_{x, x_m^+, x_k^- \overset{\text{i.i.d.}}{\sim} \mathrm{P}_X} -f(x)^\top f(x_m^+) 
        \nonumber\\
        &+ \lim_{K \to \infty}  \log\big(\frac{1}{K} e^{f(x)^\top f(x_m^+)} + \frac{1}{K}\sum_{k \in [K]} e^{f(x)^\top f(x^-_k)}\big)
	\nonumber\\
	&= \mathbb{E}_{x, x^+ \overset{\text{i.i.d.}}{\sim} \mathrm{P}_X} -f(x)^\top f(x^+) 
        \nonumber\\
        &+ \mathbb{E}_{x \sim \mathrm{P}_X} \log \big(\mathbb{E}_{x^- \sim \mathrm{P}_X} e^{f(x)^\top f(x^-)}\big).
\end{align}

\begin{proof}[Proof of Theorem \ref{prop::RevNCE}]
	\begin{align}
		&\lim_{M, K \to \infty} \Delta \mathcal{R} (\mathcal{L}_{\mathrm{RevNCE}}; f)
		\nonumber\\
		&= \lim_{K,M \to \infty} \mathbb{E}_{x, x_m^+, x_k^- \overset{\text{i.i.d.}}{\sim} \mathrm{P}_X} 
            \nonumber\\
            &\qquad\qquad\cdot\frac{1}{K}\sum_{k \in [K]} -\log \frac{\frac{1}{M}\sum_{m \in [M]} e^{f(x)^\top f(x_m^+)}}{e^{f(x)^\top f(x_k^-)}}
		\nonumber\\
		&= \frac{1}{K}\sum_{k \in [K]} f(x)^\top f(x_k^-) 
            \nonumber\\
		&- \mathbb{E}_{x, x_m^+, x_k^- \overset{\text{i.i.d.}}{\sim} \mathrm{P}_X} \lim_{K,M \to \infty} \frac{1}{K}\sum_{k \in [K]} \log \Big(\frac{1}{M}\sum_{m \in [M]} e^{f(x)^\top f(x_m^+)}\Big)
		\nonumber\\
		&= \mathbb{E}_{x, x^- \overset{\text{i.i.d.}}{\sim} \mathrm{P}_X} f(x)^\top f(x_k^-) 
            \nonumber\\
		&- \mathbb{E}_{x \sim \mathrm{P}_X} \log \Big(\mathbb{E}_{x^+ \sim \mathrm{P}_X}  e^{f(x)^\top f(x^+)}\Big)
		\nonumber\\
		&= \log K - \lim_{M, K \to \infty} \Delta \mathcal{R}(\mathcal{L}_{\mathrm{InfoNCE}}; f). 
	\end{align}
\end{proof}

\subsection{Theoretical Results and Proofs for Asymmetric Label Noise}\label{sec::asym}

Next, we show the results under asymmetric label noise.
\begin{assumption}[Asymmetric label noise]\label{ass::asymnoise}
	For noise rates $\gamma_i \in (0,1)$, $i \in [C]$, we assume that
	there holds
	$q_i(i)=1-\gamma_i$ and $q_j(i) = \gamma_{ij} \geq 0$ for all $j\neq i$.
\end{assumption}

\begin{theorem}\label{thm::asymrisk}
	Assume that the input data is class balanced, i.e. $\pi_i = \tilde{\pi}_i = 1/C$ for $i \in [C]$. Then under Assumption \ref{ass::asymnoise}, if we have $\sum_{i \in [C]} \gamma_{iu}^2 = c_1(\boldsymbol{\gamma})$ and 
	$\sum_{i \in [C]} \gamma_{iu}\gamma_{iu^+} = c_2(\boldsymbol{\gamma})$ for all $u \neq u^+ \in [C]$, then for an arbitrary contrastive loss $\mathcal{L}(\boldsymbol{x}; f)$, where $c_1(\boldsymbol{\gamma})$ and $c_2(\boldsymbol{\gamma})$ are constants related to $\boldsymbol{\gamma} := (\gamma_{ij})_{i,j \in [C]}$, there holds
        {
	\begin{align}\label{eq::riskconsistencya}
		\widetilde{\mathcal{R}}(\mathcal{L}; f) 
		&= \big(c_1(\boldsymbol{\gamma}) - c_2(\boldsymbol{\gamma})\big) \mathcal{R}(\mathcal{L}; f) + C\cdot c_2(\boldsymbol{\gamma}) \Delta\mathcal{R}(\mathcal{L}; f),
	\end{align}
 }
	where 
        {
	\begin{align*}
		\Delta \mathcal{R}(\mathcal{L}; f) 
		&:= \mathbb{E}_{\substack{c^+, \{c_m^+\}_{m=1}^M \sim \boldsymbol{\pi}^{M+1}\\ \{c_k^-\}_{k=1}^K \sim \boldsymbol{\pi}^{K}}}
		\mathbb{E}_{\substack{x \sim \rho_{c^+},\ x_m^+ \sim \rho_{c_m^+}, m \in [M]\\ x^-_k \sim \rho_{c_k^-}, \, k \in [K]}}
		\mathcal{L}(\boldsymbol{x}; f).
	\end{align*}
 }
\end{theorem}

\begin{proof}[Proof of Theorem \ref{thm::asymrisk}]
        {
	\begin{align*}
		\widetilde{\mathcal{R}}(\mathcal{L}; f) 
		&= \sum_{i \in [C]} a(i)^{-1} \sum_{u\in[C]} 	\sum_{u^+\in[C]} \pi_u \pi_{u^+} q_u(i) q_{u^+}(i) B(u,u^+)
		\nonumber\\
		&= \sum_{i \in [C]} a(i)^{-1} (1-\gamma_i)^2 \pi_i^2 B(i,i)
		\nonumber\\
            &+ \sum_{i \in [C]} a(i)^{-1} \sum_{u \neq i} \gamma_{iu} (1-\gamma_i) \pi_{u}\pi_i B(u,i)
		\nonumber\\
		&+ \sum_{i \in [C]} a(i)^{-1} \sum_{u^+ \neq i} \gamma_{iu^+} (1-\gamma_i) \pi_{u^+}\pi_i B(u,u^+)
		\nonumber\\
		&+ \sum_{i \in [C]} a(i)^{-1} \sum_{u \neq i} \sum_{u^+ \neq i} \pi_u \pi_{u^+} \gamma_{iu}\gamma_{iu^+} B(i,u^+)
		\nonumber\\
		&= \sum_{u \in [C]} a(u)^{-1} (1-\gamma_u)^2 \pi_u^2 B(u,u)
		\nonumber\\
            &+ \sum_{u^+ \in [C]} a(u^+)^{-1} \sum_{u \neq u^+} \gamma_{u^+u} (1-\gamma_{u^+}) \pi_{u}\pi_{u^+} B(u,u^+)
		\nonumber\\
		&+ \sum_{u \in [C]} a(u)^{-1} \sum_{u^+ \neq u} \gamma_{uu^+} (1-\gamma_u) \pi_{u^+}\pi_u B(u,u^+)
		\nonumber\\
		&+ \sum_{i \in [C]} a(i)^{-1} \sum_{u \neq i} \pi_u^2 \gamma_{iu}^2 B(u,u)
		\nonumber\\
            &+ \sum_{i \in [C]} a(i)^{-1} \sum_{u \neq i} \sum_{u^+ \neq u,i} \pi_u \pi_{u^+} \gamma_{iu}\gamma_{iu^+} B(u,u^+)
		\nonumber\\
		&= \sum_{u \in [C]} a(u)^{-1} (1-\gamma_u)^2 \pi_u^2 B(u,u)
		\nonumber\\
            &+ \sum_{u^+ \in [C]} \sum_{u \neq u^+} a(u^+)^{-1} \gamma_{u^+u} (1-\gamma_{u^+}) \pi_{u}\pi_{u^+} B(u,u^+)
		\nonumber\\
		&+ \sum_{u \in [C]} \sum_{u^+ \neq u} a(u)^{-1} \gamma_{uu^+} (1-\gamma_u) \pi_u \pi_{u^+} B(u,u^+)
		\nonumber\\
		&+ \sum_{u \in [C]} \sum_{i \neq u} a(i)^{-1} \gamma_{iu}^2 \pi_u^2 B(u,u)
		\nonumber\\
            &+ \sum_{u \in [C]} \sum_{u^+ \neq u} \sum_{i \neq u, u^+} a(i)^{-1} \gamma_{iu}\gamma_{iu^+} \pi_u \pi_{u^+} B(u,u^+).
	\end{align*}
	Because $a(i) = \tilde{\pi}_i = 1/C$ for all $i \in [C]$, we have
	{
        \begin{align}
		\widetilde{\mathcal{R}}(\mathcal{L}; f)
		&= C \sum_{u \in [C]} (1-\gamma_u)^2 \pi_u^2 B(u,u)
		\nonumber\\
            &+ C \sum_{u^+ \in [C]} \sum_{u \neq u^+} \gamma_{u^+u} (1-\gamma_{u^+}) \pi_{u}\pi_{u^+} B(u,u^+)
		\nonumber\\
		&+ C \sum_{u \in [C]} \sum_{u^+ \neq u} \gamma_{uu^+} (1-\gamma_u) \pi_u \pi_{u^+} B(u,u^+)
		\nonumber\\
		&+ C \sum_{u \in [C]} \sum_{i \neq u} \gamma_{iu}^2 \pi_u^2 B(u,u)
		\nonumber\\
		&+ C \sum_{u \in [C]} \sum_{u^+ \neq u} \sum_{i \neq u, u^+} \gamma_{iu}\gamma_{iu^+} \pi_u \pi_{u^+} B(u,u^+)
		\nonumber\\
		&= C \sum_{u \in [C]} \Big[(1-\gamma_u)^2 +\sum_{i \neq u} \gamma_{iu}^2\Big] \pi_u^2 B(u,u)
		\nonumber\\
		&+ C  \sum_{u \in [C]} \sum_{u^+ \neq u} \Big[\sum_{i \neq u, u^+} \gamma_{iu}\gamma_{iu^+} + \gamma_{u^+u} (1-\gamma_{u^+}) 
            \nonumber\\
		&\qquad\qquad\qquad\qquad+ \gamma_{uu^+} (1-\gamma_u)\Big] \pi_u \pi_{u^+} B(u,u^+)
		\nonumber\\
		&= C \sum_{u \in [C]} \Big(\sum_{i \in [C]} \gamma_{iu}^2\Big) \pi_u^2 B(u,u) 
            \nonumber\\
		&+ C \sum_{u \in [C]} \sum_{u^+ \neq u} \Big(\sum_{i \in [C]} \gamma_{iu}\gamma_{iu^+}\Big) \pi_u \pi_{u^+} B(u,u^+).
	\end{align}
 }
 }
	By assumption, $\sum_{i \in [C]} \gamma_{iu}^2 = c_1(\boldsymbol{\gamma})$ and 
	$\sum_{i \in [C]} \gamma_{iu}\gamma_{iu^+} = c_2(\boldsymbol{\gamma})$ for all $u \neq u^+ \in [C]$. Thus we have
        {
        \begin{align}
		\widetilde{\mathcal{R}}(\mathcal{L}; f)
		&= C\cdot c_1(\boldsymbol{\gamma}) \sum_{u \in [C]} \pi_u^2 B(u,u) 
            \nonumber\\
		&+ C\cdot c_2(\boldsymbol{\gamma}) \sum_{u \in [C]} \sum_{u^+ \neq u} \pi_u \pi_{u^+} B(u,u^+)
		\nonumber\\
		&= C \big(c_1(\boldsymbol{\gamma}) - c_2(\boldsymbol{\gamma})\big) \sum_{u \in [C]} \pi_u^2 B(u,u) 
		\nonumber\\
		&+ C \cdot c_2(\boldsymbol{\gamma}) \sum_{u, u^+ \in [C]}\pi_u \pi_{u^+} B(u,u^+)
		\nonumber\\
		&= \big(c_1(\boldsymbol{\gamma}) - c_2(\boldsymbol{\gamma})\big) \mathcal{R}(\mathcal{L}; f) + C \cdot c_2(\boldsymbol{\gamma}) \Delta\mathcal{R}(\mathcal{L}; f).
	\end{align}
        }
\end{proof}

In Theorem \ref{thm::noisetolerant}, we give the general condition for an arbitrary contrastive loss function to be noise tolerant under asymmetric label noise. 
Comparing with \cite{ghosh2017robust}, our theoretical framework requires the same noise rate $1 - \gamma_u > \gamma_{ui}$ for all $i \neq u$, and $u,i \in [C]$.

\begin{theorem}\label{thm::asymtolerant}
	Assume that the input data is class balanced, and there exists a constant $A \in \mathbb{R}$ such that $\Delta \mathcal{R}(\mathcal{L}; f) = A$. Then under Assumption \ref{ass::asymnoise}, contrastive loss $\mathcal{L}$ is noise tolerant if $1-\gamma_u > \gamma_{ui}$ for all $i \neq u$, $u,i \in [C]$.
\end{theorem}

\begin{proof}[Proof of Theorem \ref{thm::asymtolerant}]
	When $\Delta \mathcal{R}(\mathcal{L}; f) = A$, we have
	{
        \begin{align}
		\widetilde{\mathcal{R}}(\mathcal{L}; f) 
		=  \big(c_1(\boldsymbol{\gamma}) - c_2(\boldsymbol{\gamma})\big) \mathcal{R}(\mathcal{L}; f) + CA \cdot c_2(\boldsymbol{\gamma}).
	\end{align}
 }
	Then we calculate $c_1(\boldsymbol{\gamma}) - c_2(\boldsymbol{\gamma})$.
        {
	\begin{align*}
		&c_1(\boldsymbol{\gamma}) - c_2(\boldsymbol{\gamma})
		\nonumber\\
            &= \sum_{i \in [C]} \gamma_{iu}^2 -  \sum_{i \in [C]} \gamma_{iu} \gamma_{iu^+}
		\nonumber\\
		&= \frac{1}{2} \sum_{i \in [C]} \big(\gamma_{iu}^2 + \gamma_{iu^+}^2 - 2\gamma_{iu}\gamma_{iu^+}\big)
		\nonumber\\
		&= \frac{1}{2} \sum_{i \in [C]} (\gamma_{iu} - \gamma_{iu^+})^2
		\nonumber\\
		&= \frac{1}{2} \Big[\sum_{i \neq u, u^+} (\gamma_{iu} - \gamma_{iu^+})^2 
            \nonumber\\
            &\qquad
            + (1-\gamma_u-\gamma_{uu^+})^2 + (1-\gamma_{u^+}-\gamma_{u^+u})^2\Big].
	\end{align*}
 }
	If $1-\gamma_u > \gamma_{ui}$ for all $i \neq u$, $u,i \in [C]$, then $c_1(\boldsymbol{\gamma}) - c_2(\boldsymbol{\gamma}) > 0$. Suppose $f^*$ is the minimizer of $\widetilde{\mathcal{R}}(\mathcal{L}; f)$, and thus we have when $1-\gamma_u > \gamma_{ui}$,
	{
        \begin{align*}
		&\mathcal{R}(\mathcal{L}; f) - \mathcal{R}(\mathcal{L}; f^*)
		\nonumber\\
		&= \frac{1}{\big(c_1(\boldsymbol{\gamma}) - c_2(\boldsymbol{\gamma})\big)} \Big[\widetilde{\mathcal{R}}(\mathcal{L}; f^*) - \widetilde{\mathcal{R}}(\mathcal{L}; f)\Big] 
            \nonumber\\
		&< 0,
	\end{align*}
        }
	that is, $f^*$ is also the minimizer of $\mathcal{R}(\mathcal{L}; f)$.
\end{proof}

\end{document}